\documentclass[10pt,journal,cspaper,compsoc]{IEEEtran}

\usepackage{bbold}
\usepackage{amsmath}
\usepackage{mathrsfs}
\usepackage[ruled,vlined,linesnumbered]{algorithm2e}
\providecommand{\SetAlgoLined}{\SetLine}

\usepackage{algpseudocode}
\usepackage[final]{graphicx}
\usepackage{epstopdf}
\usepackage{subfigure}
\usepackage{amsthm}
\usepackage{amsfonts}
\usepackage{bm}

\ifCLASSOPTIONcompsoc
\else
\fi

\ifCLASSINFOpdf
\else
\fi

\begin{document}
%
\title{A Convex Sparse PCA for Feature Analysis}
%
%
%
%

\author{Xiaojun Chang,
        Feiping Nie,
        Yi Yang,
        and~Heng~Huang
\IEEEcompsocitemizethanks{\IEEEcompsocthanksitem Xiaojun Chang and Yi Yang are with School of Information Technology and Electrical Engineering, The University of Queensland, Australia. (E-mail: x.chang@uq.edu.au; yi.yang@uq.edu.au).\protect\\

\IEEEcompsocthanksitem Feiping Nie and Heng Huang are with Department of Computer Science and Engineering, University of Texas at Arlington. (E-mail: feipingnie@gmail.com, heng@uta.edu).}
\thanks{}}

%
%

\markboth{Journal of \LaTeX\ Class Files,~Vol.~6, No.~1, January~2007}%
{Shell \MakeLowercase{\textit{et al.}}: Bare Demo of IEEEtran.cls for Computer Society Journals}
%


\IEEEcompsoctitleabstractindextext{%
\begin{abstract}
Principal component analysis (PCA) has been widely applied to dimensionality reduction and data pre-processing for different applications in engineering, biology and social science. Classical PCA and its variants seek for linear projections of the original variables to obtain a low dimensional feature representation with maximal variance. One limitation is that it is very difficult to interpret the results of PCA. In addition, the classical PCA is vulnerable to certain noisy data. In this paper, we propose a convex sparse principal component analysis (CSPCA) algorithm and apply it to feature analysis. First we show that PCA can be formulated as a low-rank regression optimization problem. Based on the discussion, the $l_{2,1}$-norm minimization is incorporated into the objective function to make the regression coefficients sparse, thereby robust to the outliers. In addition, based on the sparse model used in CSPCA, an optimal weight is assigned to each of the original feature, which in turn provides the output with good interpretability. With the output of our CSPCA, we can effectively analyze the importance of each feature under the PCA criteria. The objective function is convex, and we propose an iterative algorithm to optimize it.  We apply the CSPCA algorithm to feature selection and  conduct extensive experiments on six different benchmark datasets. Experimental results demonstrate that the proposed algorithm outperforms state-of-the-art unsupervised feature selection algorithms.
\end{abstract}

}

\maketitle

\IEEEdisplaynotcompsoctitleabstractindextext

%
\IEEEpeerreviewmaketitle

\section{Introduction}
In many machine learning and data mining applications, such as face recognition \cite{eigenfaces} \cite{PCAFR}, conceptual indexing \cite{ConceptualIndexing}, collaborative filtering \cite{collaborativefiltering}, the dimensionality of the input data is usually very high. It is computationally expensive to analyze the high-dimensional data directly. Meanwhile, the noise in a representation may dramatically increase as the dimensionality is getting high \cite{lineardp} \cite{fsclustering} \cite{vpca}. To improve the efficiency and accuracy, researchers have demonstrated that dimensionality reduction is one of the most effective approaches for data analysis, and plays a significant role in data mining. Because of its simplicity and effectiveness, Principal Component Analysis (PCA) has been widely applied to various applications. The goal of PCA is to find a projection matrix that maximizes the variance of the samples after the projection, while preserving the structure of the original dataset as much as possible.

PCA seeks a linear projection for the original high-dimensional feature vectors so as to obtain a low dimensional representation of data, which captures as much information as possible. One may obtain principal components (PCs) by performing singular value decomposition (SVD) of the original data matrix and choose the first $k$ PCs to represent the data, which is a more compact feature representation. There are two main reasons why PCA usually obtains good performance in the real world applications: (1) all the PCs are uncorrelated; (2) minimal information loss is guaranteed by the fact that PCs sequentially capture maximum variability among columns of data matrix. Nevertheless, PCA still has some inherent drawbacks, which this paper will address.

One problem of the classical PCA is that each PC is obtained by a linear combination of original variables and loadings are normally non-zero, which makes it often difficult to interpret the results. To address this problem, Hui Zou etal. integrate the lasso penalty \cite{lasso}, which is a variable selection technique, into the regression criterion in \cite{spca}. In their paper, they propose a new approach for estimating PCs with sparse loadings, sparse principal component analysis (SPCA). Lasso penalty is implemented via elastic net, which is a generalization of lasso proposed in \cite{elastic}. However, their algorithm is non-convex and it is difficult to obtain the global optima. Thus the performance may vary dramatically with different local optima.

Another drawback of the classical PCA methods is that they are least square estimation approaches, which are commonly known not to be robust in the sense that outlying measurements can arbitrarily skew the solution from the desired solution \cite{robust}. To make PCA robust to outliers, Xu etal. \cite{rpcaso} propose to recover a low-rank matrix from highly corrupted measurements. It has been experimentally demonstrated in \cite{robust} that robust PCA gains promising performance on noisy data analysis. However, despite of its robustness to the outliers, the algorithm proposed in \cite{robust} is transductive, and is not able to deal with the out-of-sample data which are unseen during the training phrase. It is very restrictive to have all the data beforehand. Therefore, the robust PCA algorithm proposed in \cite{rpcaso} is less practical for many real world application.

In this paper, we propose a novel convex sparse PCA for feature analysis. It has been demonstrated in \cite{spca} that the sparse model is a good measure for feature analysis, especially for feature weighting. We therefore impose the $l_{2,1}$-norm on the regression coefficient so as to make our algorithm able to evaluate the importance of each feature. Besides, we adopt the $l_{2,1}$-norm based loss function, which is robust to the outliers, to achieve robust performance. Different from \cite{robustPCAMAYI}, our algorithm is inductive and can be directly used to map the unseen data which are outside the training set. We name the proposed algorithm Convex Sparse PCA (CSPCA). 
The main contributions of this paper can be summarized as follows:

\begin{enumerate}
\item We have theoretically proved the equivalence of the classical PCA and low rank regression.

\item The proposed algorithm combines the recent advances of sparsity and robust PCA into a joint framework to leverage the mutual benefit. To the best of our knowledge, this is the first convex sparse and robust PCA algorithm, which ensures our algorithm always achieves the global optima.

\item Different from the existing robust PCA algorithms \cite{robustPCAMAYI} \cite{RPCAOP}, which can only deal with the in-sample data, our algorithm is capable of mapping the data which are unseen during the training phase.

\item We propose an effective iterative algorithm to optimize the objective function, which simultaneously optimizes the $l_{2,1}$-norm minimization and the trace norm minimization.

\end{enumerate}

The rest of this paper is organized as follows. We briefly review related work on PCA, sparse PCA and robust PCA in Section 2. Then we elaborate the formulation of our method in Section 3, followed by the proposed solution in Section 4. Extensive experiments are conducted in Section 5 to evaluate performance of the proposed algorithm. Section 6 concludes this paper.
\section{Related Work}

In this section, we briefly review three related topics of our work, including the classical PCA, sparse PCA and robust PCA.

To begin with, we first define the terms and notations which will be frequently used in this paper. (1) data matrix denoted by $X = [x_1, x_2, \cdots , x_n]$ where $x_i \in \mathbb{R}^d ( 1 \leq i \leq n)$ is the $i$-th datum and $n$ is the total number of the samples; (2) projection matrix denoted by $W$; (3) the Frobenius norm denoted by $\|X\|_F$; (4) the trace norm denoted by $\|W\|_*$.

\subsection{The Classical PCA}

The classical PCA is a statistical technique for dimensionality reduction. Classical PCA techniques, also known as Karhunen-Loeve methods, look for a dimensionality reducing linear projection that maximizes the total scatter of all projected data points. To be more specific, PCA computes the PCs by performing eigen-value decomposition of covariance of the convariance matrix of all training data. In general, the entries of corresponding PCs are dense and non-zero. The objective function of classical PCA is 

\begin{equation}\nonumber
\max_{W^TW=I} Tr(W^tXXW),
\end{equation}
where $Tr(\cdot)$ denotes trace operator.

\subsection{Sparse PCA}

A common limitation of the classical PCA is the lack of interpretability. All principal components are a linear combination of variables and most of the factor coefficients are non-zero. To get more interpretable results, sparse PCA is proposed, which leads to reduced computation time and improved generalization. There are numerous implementations of sparse PCA in the literature \cite{spcasemi} \cite{fullregSPCA} \cite{spectralSPCA} \cite{spcasemi} \cite{lassospca} \cite{sparsepcalowrank}. The objectives of all the methods aim to reduce the dimensionality reduction and the number of explicitly used variables. A straightforward way is to manually set factor coefficients with values below a threshold to zero. This simple and naive thresholding method is often adopted in various applications. Nevertheless, it could be potentially misleading in different aspects. Jolliffe etal. propose SCoTLASS to obtain modified principal components with possible zero factor coefficients \cite{lassospca}. Lasso \cite{lasso} has shown to be a effective variable selection method, which has been shown effective in a variety of applications. To further improve lasso, Zou etal. propose the elastic net in \cite{elastic} for sparsity based mining. Based on the fact that PCA can be reformulated as regression-type optimization problem, Zou etal. \cite{spca} propose sparse PCA (SPCA) for estimating PCs with sparse factor coefficients, which can be formulated as follows:

\begin{equation}\nonumber
\begin{aligned}
&\min_{A, B} \sum_{i=1}^n \|x_i - AB^Tx_i\|^2 + \lambda \sum_{j=1}^k \|\beta _j \|^2 + \sum_{j=1}^k \lambda _{1,j} \|\beta _j \|_1, \\
& ~~~ s.t. ~A^TA=I
\end{aligned}
\end{equation}
where $\beta$ is lasso estimates. All $k$ components share the same $\lambda$ and different $\lambda_{1,j}$'s are allowed for penalizing the loadings of different principal components. Although the algorithm has good performance and attracted more and more attention, it is non-convex and difficult to find the global optima.

\subsection{Robust PCA}

The goal of robust PCA is to recover a low-rank matrix $D$ from highly corrupted measurements $X = D + E$. The errors $E$ are supposed to be sparsely supported. Motivated by recent research on the robust solution of over-determined linear systems of equations in the presence of arbitrary but sparse errors and computing low-rank matrix solutions to underdetermined linear equations, John etal. \cite{robustPCAMAYI} propose exact recovery of corrupted low-rank matrices by convex optimization. A straightforward solution to robust PCA is to seek the matrix with the lowest rank that could have generated the data under the constraint of sparse errors. The objective function of robust PCA is formulated as follows:

\begin{equation}
\label{robustPCA1}
\min_{D} \|X - D \|_0 + \gamma rank(D)
\end{equation}

However, since Eq. \eqref{robustPCA1} involves $l_0$-norm, the objective function is highly non-convex and it is difficult to find an efficient solution. To obtain a tractable optimization problem, it is nature to replace the $l_0$-norm with $l_1$-norm and the rank with the trace norm. The objective function can be rewritten as:

\begin{equation}
\min_{D} \|X - D\|_1 + \gamma \|D\|_*
\end{equation}

To make the objective function robust to outliers, we further replace $l_1$-norm with $l_{2,1}$-norm as $l_{2,1}$-norm is indicated to make the objective function robust to outliers in \cite{RPCAOP}. The objective function arrives at:

\begin{equation}
\min_{D} \|X - D\|_{2,1} + \gamma \|D\|_*
\end{equation}

Although the robust PCA has attracted much research attention in recent years, it still has a major limitation. As the robust PCA is transductive, despite of its good performance, it cannot be applied to out-of-sample problems. In other words, it cannot map the data, which are outside the training set, into the low dimensional subspace.

\section{The Proposed method}

In this section, we first demonstrate the equivalence of PCA and regression, followed by illustrating the formulation of the convex sparse PCA method. Then we describe a detailed approach to solve the objective function.

\subsection{The Equivalence of Classical PCA and Regression}

The proposed CSPCA is designed upon our recent finding that the classical PCA can be reformulated as a regression problem. This conclusion provides us with new insights of PCA in a different perspective, and enables us to design the new convex sparse PCA algorithm. We begin the following theorem.

\newtheorem{theorem}{Theorem}
\begin{theorem}
The classical PCA can be reformulated as a low-rank regression optimization problem as follows:
\begin{equation}
\min_{rank(W)=k} \|W^TX - X\|_F^2
\label{obj1}
\end{equation}
\end{theorem}
\begin{proof}
As we have the constraint $rank(W) = k$, we can easily write $W = BA^T$, where $A \in \mathbb{R}^{d \times k}$ is an orthogonal matrix, $B \in \mathbb{R}^{d \times k}$ and the rank of both $A$ and $B$ are $k$. The above objective function can be rewritten as follows:
\begin{equation}
\begin{aligned}
 & \min_{A,B \in \mathbb{R}^{d \times k}, A^TA=I} \|AB^TX - X \|_F^2 \\
 = & \min_{A,b \in \mathbb{R}^{d \times k}, A^TA=I} Tr(B^TXX^TB) - 2Tr(B^TXX^TA).
\end{aligned}
\label{der1}
\end{equation}

By setting the derivatives of \eqref{der1} w.r.t $B$ to zero, we have:

\begin{equation}
XX^TB = XX^TA
\end{equation}
By denoting $X=U\Sigma V^T$, $U^\perp$ as orthogonal complement standard basis vectors of $U$ and $B=U\alpha + U^\perp \beta$ ($\beta$ is an arbitrary vector), we have the following mathematical deduction:
\begin{equation}
\begin{aligned}
& XX^TB = XX^TA \\
\Rightarrow & U\Sigma ^2U^T(U\alpha + U^\perp \beta) = U \Sigma^2 U^TA \\
\Rightarrow & U\Sigma^2\alpha = U\Sigma^2U^TA \\
\Rightarrow & \alpha = U^TA.
\end{aligned}
\end{equation}
Hence, we have $B=UU^TA + U^\perp\beta$. By incorporating $B$ into \eqref{obj1}, we obtain:
\begin{equation}
\begin{aligned}
& \min_{\beta, A^TA=I} \|AA^TUU^TX + A\beta ^T (U^\perp)^TX - X \|_F^2 \\
\Rightarrow & \min_{\beta , A^TA=I} \|AA^TUU^TU\Sigma V^T + A\beta^T(U^\perp)^TU\Sigma V^T - X\|_F^2 \\
\Rightarrow & \min_{A^TA=I} \|AA^TX - X\|_F^2.
\end{aligned}
\end{equation}
Hence, we have $A=U_1Q$, where $Q$ is an arbitrary orthogonal matrix. And we can get
\begin{equation}
B = UU^TA + U^\perp\beta = UU^TU_1Q + U^\perp\beta = U_1Q + U^\perp \beta
\end{equation}
With the obtained $A$ and $B$, we can get:
\begin{equation}
\begin{aligned}
W & = AB^T = U_1Q(U_1Q + U^\perp\beta)^T \\
 & = U_1U_1^T + U_1Q\beta ^T {U^\perp}^T
\end{aligned}
\end{equation}
The projected samples can be obtained as follows:
\begin{equation}
\begin{aligned}
W^TX & = AB^TX = U_1U_1^TU\Sigma V^T + U_1Q\beta^T{U^\perp}^TU\Sigma V^T \\
 & = U_1\Sigma _1V_1,
\end{aligned}
\end{equation}
which is equivalent to projected samples obtained by classical PCA.
\end{proof}

\emph{\textbf{The connection between the stated Theorem 1 and Theorem 2 in \cite{spca}:}} Zou etal. claim that when $\lambda > 0$, PCA problem can be transformed into a regression-type problem by the following Theorem:
\begin{theorem}
For any $\lambda > 0$, let
\begin{equation}
\begin{aligned}
& (\hat{\alpha}, \hat{\beta}) = \min_{\alpha , \beta} \sum_{i=1}^n \|x_i - \alpha \beta ^T x_i \|^2 + \lambda \|\beta\|^2 \\
& s.t.~\|\alpha\|^2 = 1.
\end{aligned}
\end{equation}
\end{theorem}
In the above theorem, $\beta$ is the lasso estimates and $\hat{\beta}$ $\propto$ the space of PCA. Compared with Theorem 2 proposed in \cite{spca}, our contribution is that we prove that when $\lambda = 0$, PCA problem is completely equivalent to a regression-type problem.

\subsection{The Proposed Objective Function}

In this section, we detail the proposed objective function of SCPCA. Motivated by previous work \cite{l21norm}, which demonstrate that $l_{2,1}$-norm of $W$ is capable of making $W$ sparse, we propose our sparse PCA algorithm as follows:

\begin{equation}
\min_{rank(W)=k} \|(W^TX - X)^T \|_2^2 + \alpha \|W\|_{2,1},
\end{equation}
where $l_{2,1}$-norm of $W$ is defined as

\begin{equation}\nonumber
\|W\|_{2,1} = \sum_{i=1}^d \sqrt{\sum_{j=1}^d W_{ij}^2}.
\end{equation}

In the above function, $\|W^TX - X\|_2^2$ is the most commonly used least square loss function and is mathematically tractable and easily implemented. However, there are still some existing issues which need to take into further consideration. For example, it is well known that the least square loss function is very sensitive to outliers \cite{l21norm}. To address this issue, it is important for us to adopt a more robust loss function in the objective. In \cite{l21norm}, Nie etal. demonstrate that $l_{2,1}$-norm is more capable of dealing with the noisy data.

Therefore, our proposed algorithm is rewritten as follows:

\begin{equation}
\min_{rank(W) = k} \|(W^TX - X)^T\|_{2,1} + \alpha \|W\|_{2,1}
\end{equation}

In the above formulation, the loss function $\|W^TX - X\|_{2,1}$ is robust to outliers, as proven in \cite{l21norm}. Meanwhile, $\|W\|_{2,1}$ in the regularization term is guaranteed to make $W$ sparse in rows.

Next, we first give the definition of trace norm. The trace norm of $W$ is defined as 
\begin{equation}
\|W\|_{*} = Tr(WW^T)^{\frac{1}{2}}.
\end{equation}
Following the work in \cite{robustPCAMAYI} \cite{ratiorules}, we restrict $W$ to be a low rank matrix. To have the problem tackable, we propose to minimize the trace norm of $W$, which is the convex hull of the rank of $W$. The objective function of the proposed algorithm is then given by:

\begin{equation}
\min_W \|(W^TX - X)^T\|_{2,1} + \alpha \|W\|_{2,1} + \beta \|W\|_{*}
\label{finalobj}
\end{equation}

Compared with directly minimizing the rank of $W$, our proposed objective function as shown in \eqref{finalobj} is convex. We therefore name the proposed algorithm convex sparse PCA (CSPCA). Different from the previous robust PCA algorithms \cite{robustPCAMAYI} \cite{RPCAOP}, the proposed algorithm is inductive, and able to deal with the out-of-sample data which are unseen in the training phase. Given a new testing data point $x_t$, we can get its low dimensional representation by $W^Tx_t$ directly.

\subsection{Optimization}

As can be seen from Eq. \eqref{finalobj}, the proposed algorithm involves the $l_{2,1}$-norm, which is non-smooth and cannot be solved in a closed form. Hence, we proposed to solve this problem as follows.

For an arbitrary matrix $A$, we denote $A = [A^1, \cdots , A^d]$, where $d$ is the number of features. By setting the derivatives w.r.t $W$ to zero, we have

\begin{equation}\nonumber
XD_1X^TW + \alpha D_2W + \beta D_3W = XD_1X.
\end{equation}

Then we have

\begin{equation}
 W = (XD_1X^T + \alpha D_2 + \beta D_3)^{-1} (XD_1X^T),
\label{devri}
\end{equation}
where $D_1$, $D_2$ and $D_3$ are diagonal matrices defined as follows.

$
D_1 = \begin{bmatrix}
\frac{1}{2\|e^1\|_2} & & \\
  &  \ddots &  \\
  &   &  \frac{1}{2\|e^d\|_2} \\
\end{bmatrix}$,

where $E = (W^TX-X)^T$.

$
D_2 = \begin{bmatrix}
\frac{1}{2\|w^1\|_2} & & \\
  &  \ddots &  \\
  &   &  \frac{1}{2\|w^d\|_2} \\
\end{bmatrix}$

$D_3 = \frac{1}{2}(WW^T)^{-\frac{1}{2}}$

Based on the above mathematical deduction, we propose an iterative algorithm to optimize the objective function Eq. \eqref{finalobj}, which is summarized in Algorithm 1. In each iteration, $E$, $D_1$, $D_2$ and $D_3$ are updated by the current $W$, and then $W$ is updated based on the current calculated $E$, $D_1$, $D_2$ and $D_3$. Once $W$ is obtained and a new data point $x_i$, we get the projected representation by computing $W^Tx_i$. As the project matrix $W$ is sparse, it actually assigned a weight to each feature dimension and thus can be used for feature analysis. The importance score of each feature can be computed by $\|w^i\|_2(1\leq i \leq d)$. Then we can rank each feature according to this score. In this sense, $W$ can be readily used for feature selection and we only select the top $k$ features based on the score $\|w^i\|_2(1\leq i \leq d)$.

\begin{algorithm}
\caption{Algorithm to solve the problem in \eqref{finalobj}}
\SetAlgoLined
\KwData{Data matrix $X$ \\
~~~~~~~~Parameters $\alpha$, $\beta$}
\KwResult{W}
Set $t$ = 0 \;
Initialize $W_0 \in \mathbb{R}^{d \times c}$ randomly \;
\Repeat{Converence}{
Compute $E_t$ according to $E_t = (W_t^TX - X)^T$ \;
Compute the diagonal matrix $D_{1t}$ as follows:
\begin{equation}\nonumber
D_{1t} = \begin{bmatrix}
\frac{1}{2\|e_t^1\|_2} & & \\
  &  \ddots &  \\
  &   &  \frac{1}{2\|e_t^d\|_2} \\
\end{bmatrix};
\end{equation} \\
Compute the diagonal matrix $D_{2t}$ as follows:
\begin{equation}\nonumber
D_{2t} = \begin{bmatrix}
\frac{1}{2\|w_t^1\|_2} & & \\
  &  \ddots &  \\
  &   &  \frac{1}{2\|w_t^d\|_2} \\
\end{bmatrix};
\end{equation} \\
Compute $D_{3t}$ according to 
\begin{equation}\nonumber
D_{3t} = \frac{1}{2}(W_tW_t^T)^{-\frac{1}{2}};
\end{equation} \\
Update $W_{t+1}$ according to 
\begin{equation}\nonumber
W_{t+1} = (XD_{1t}X^T + \alpha D_{2t} + \beta D_{3t})^{-1} (XD_{1t}X^T);
\end{equation} \\
$t = t + 1$ \;
}
Return $W = W_t$. 
\end{algorithm}

\vspace{-3mm}

\subsection{Convergence Analysis}

In this section, we validate Algorithm 1 shown above. Specially, we prove that the objective function value converges to the optimal $W$ by the following theorem.

\begin{theorem}
The objective function value shown in Eq. \eqref{finalobj} monotonically decreases in each iteration until convergence using the iterative approach in Algorithm 1.
\end{theorem}

\vspace{-3mm}
\begin{proof}
According to the 8th step of Algorithm 1, it can be safely inferred that:

\vspace{-5mm}
\begin{equation}\nonumber
\begin{aligned}
W_{t+1} & = \arg \min Tr((W^TX - X)D_1(W^TX - X)) \\
& + \alpha Tr(W^TD_2W) + \beta Tr(W^TD_3W)
\end{aligned}
\end{equation}

\vspace{-4mm}
Therefore, we have:
\begin{equation}\nonumber
\begin{aligned}
& Tr((W_{t+1}^TX - X)D_1(W_{t+1}^TX - X)) \\
& + \alpha Tr(W_{t+1}^TD_2W_{t+1}) + \beta Tr(W_{t+1}D_3W_{t+1}) \\ \\
\leq & Tr((W_t^TX - X)D_1(W_t^TX - X)) \\
& + \alpha Tr(W_t^TD_2W_t) + \beta Tr(W_tD_3W_t)
\end{aligned}
\end{equation}

\begin{equation}\nonumber
\begin{aligned}
\Rightarrow & \sum_{i=1}^n \frac{\|W_{t+1}^Tx_i - x_i\|_2^2}{2\|W_{t+1}^Tx_i - x_i\|_2} + \alpha \sum_{i=1}^d \frac{\|w_{t+1}^i\|_2^2}{2\|w_t^i\|_2}  \\
& + \frac{\beta}{2} Tr(W_{t+1}^T(W_tW_t^T)^{-\frac{1}{2}}W_{t+1}) \\
\leq & \sum_{i=1}^n \frac{\|W_{t}^Tx_i - x_i\|_2^2}{2\|W_{t}^Tx_i - x_i\|_2} + \alpha \sum_{i=1}^d \frac{\|w_{t}^i\|_2^2}{2\|w_t^i\|_2}  \\
& + \frac{\beta}{2} Tr(W_{t}^T(W_tW_t^T)^{-\frac{1}{2}}W_{t}) 
\end{aligned}
\end{equation}

\begin{equation}\nonumber
\begin{aligned}
\Rightarrow & \sum_{i=1}^n \|W_{t+1}^Tx_i - x_i \|_2 - \sum_{i=1}^n \|W_{t+1}^Tx_i - x_i \|_2 \\& + \frac{\|W_{t+1}^Tx_i - x_i\|_2^2}{2\|W_{t+1}^Tx_i - x_i\|_2} + \alpha \sum_{i=1}^d \|w_{t+1}^i\|_2 - \alpha \sum_{i=1}^d \|w_{t+1}^i\|_2 \\
& + \alpha \sum_{i=1}^d \frac{\|w_{t+1}^i\|_2^2}{2\|w_t^i\|_2} + \frac{\beta}{2} Tr((W_{t+1}W_{t+1}^T)^{\frac{1}{2}}) \\ & - \frac{\beta}{2} Tr((W_{t+1}W_{t+1}^T)^{\frac{1}{2}}) + \frac{\beta}{2} Tr(W_{t+1}^T(W_tW_t^T)^{-\frac{1}{2}}W_{t+1})  \\
\leq & \sum_{i=1}^n \|W_{t}^Tx_i - x_i \|_2 - \sum_{i=1}^n \|W_{t}^Tx_i - x_i \|_2 + \frac{\|W_{t}^Tx_i - x_i\|_2^2}{2\|W_{t}^Tx_i - x_i\|_2} \\
& + \alpha \sum_{i=1}^d \|w_{t+1}^i\|_2 - \alpha \sum_{i=1}^d \|w_{t}^i\|_2 + \alpha \sum_{i=1}^d \frac{\|w_{t}^i\|_2^2}{2\|w_t^i\|_2} \\
& + \frac{\beta}{2} Tr((W_{t}W_{t}^T)^{\frac{1}{2}}) - \frac{\beta}{2} Tr((W_{t}W_{t}^T)^{\frac{1}{2}}) \\
& + \frac{\beta}{2} Tr(W_{t}^T(W_tW_t^T)^{-\frac{1}{2}}W_{t}) 
\end{aligned}
\end{equation}

\begin{equation}\nonumber
\begin{aligned}
\Rightarrow & \sum_{i=1}^n \|W_{t+1}^Tx_i - x_i\|_2 + \alpha \sum_{i=1}^d \|w_{t+1}^i\|_2 + \frac{\beta}{2} Tr((W_{t}W_{t}^T)^{\frac{1}{2}}) \\
& - \alpha (\sum_{i=1}^d \|w_{t+1}^i\|_2 - \sum_{i=1}^d \frac{\|w_{t+1}^i\|_2^2}{2\|w_t^i\|_2}) \\ & -\frac{\beta}{2} (Tr((W_{t+1}W_{t+1}^T)^{\frac{1}{2}}) - Tr(W_{t+1}^T(W_tW_t^T)^{-\frac{1}{2}})W_{t+1})) \\
\leq & \sum_{i=1}^n \|W_{t}^Tx_i - x_i\|_2 + \alpha \sum_{i=1}^d \|w_{t}^i\|_2 + \frac{\beta}{2} Tr((W_{t}W_{t}^T)^{\frac{1}{2}}) \\
& - \alpha (\sum_{i=1}^d \|w_{t}^i\|_2 - \sum_{i=1}^d \frac{\|w_{t}^i\|_2^2}{2\|w_t^i\|_2}) \\ & -\frac{\beta}{2} (Tr((W_{t}W_{t}^T)^{\frac{1}{2}}) - Tr(W_{t}^T(W_tW_t^T)^{-\frac{1}{2}})W_{t}))
\end{aligned}
\end{equation}

It has been proven in \cite{l21norm} that for arbitrary non-zero vectors ${v_t^i}|_{i=1}^r$ we have:

\begin{equation}\nonumber
\sum_i \|v_{t+1}^i\|_2 - \sum_i \frac{\|v_{t+1}^i\|_2^2}{2\|v_t^i\|_2}
\leq \sum_i \|v_{t}^i\|_2 - \sum_i \frac{\|v_{t}^i\|_2^2}{2\|v_t^i\|_2},
\end{equation}
where $r$ is any non-zero number. Thus, we can obtain the following inequality:
\begin{equation}\nonumber
\begin{aligned}
& \sum_{i=1}^n \|W_{t+1}^Tx_i - x_i \|_2 + \alpha \sum_{i=1}^d \|w_{t+1}^i\|_2 + \frac{\beta}{2} Tr((W_{t+1}W_{t+1})^{\frac{1}{2}}) \\
\leq & \sum_{i=1}^n \|W_{t}^Tx_i - x_i \|_2 + \alpha \sum_{i=1}^d \|w_{t}^i\|_2 + \frac{\beta}{2} Tr((W_{t}W_{t})^{\frac{1}{2}}) 
\end{aligned}
\end{equation}

\begin{equation}\nonumber
\begin{aligned}
\Rightarrow & \|W_{t+1}^TX - X\|_{2,1} + \alpha \|W_{t+1}\|_{2,1} + \beta \|W_{t+1}\|_{*} \\
\leq & \|W_{t}^TX - X\|_{2,1} + \alpha \|W_{t}\|_{2,1} + \beta \|W_{t}\|_{*}
\end{aligned}
\end{equation}
which indicates that the objective function value of Eq. \eqref{finalobj} monotonically decreases until converging to the optimal $W$ via the proposed approach in Algorithm 1.
\end{proof}

To step further, we prove that the proposed algorithm converges to the global optima by Theorem \ref{gloablopt}.

\begin{theorem}
\label{gloablopt}
The objective function value shown in Eq. \eqref{finalobj} converges to the global optima using Algorithm 1.
\end{theorem}

\begin{proof}
Once the objective function converges using algorithm 1 and returns $W^*$. According to Eq. \eqref{devri}, we can get the following equation:
\begin{equation}\nonumber
XD_1X^TW^* + \alpha D_2W^* + \beta D_3W^* - XD_1X = 0
\end{equation}
We can see that the derivatives w.r.t $W$ equals to zero, and we get the local solution to the objective function. Note that the proposed method is a convex problem. Hence, according to the Karush-Kuhn-Tucker (KKT) conditions, we conclude that the objective function converges to the global optima using Algorithm \ref{gloablopt}.
\end{proof}

\begin{table*}[tb]
\small
\caption{SETTINGS OF THE DATA SETS}
\centering
\begin{tabular}{|c||r|c|c|c|}
\hline
Dataset & Size(n) & No. of variables & Class Number & Number of Selected Features \\
\hline
YaleB & 2414 & 1024 & 38 & $\{500, 600, 700, 800, 900, 1000\}$ \\
\hline
ORL & 400 & 1024 & 40 &  $\{500, 600, 700, 800, 900, 1000\}$ \\
\hline
JAFFE & 213 & 676 & 10 &  $\{350, 390, 430, \cdots 590, 610, 650\}$ \\
\hline
HumanEVA & 10000 & 168 & 10   & $\{50, 60, 70, 80, 90, 100\}$ \\
\hline
Coil20 & 1440 & 1024 & 20 &  $\{170, 190, 210, 230, 250, 270, 290\}$ \\
\hline
USPS & 9298 & 256 & 10 &  $\{120, 140, 160, 180, 200, 220, 240\}$ \\
\hline
\end{tabular}
\label{setting}
\end{table*}

\section{Experiments}

In this section, we evaluate performance of the proposed algorithm, which can be applied to many applications, such as dimension reduction and unsupervised feature selection. Following previous unsupervised feature selection algorithms \cite{laplacianscore} \cite{SPEC} \cite{UDFS}, we only evaluate the performance of CSPCA for feature selection and compare with related state-of-the-art unsupervised feature selection.

\subsection{Experimental Settings}
To demonstrate the effectiveness of the proposed algorithm for feature selection, we compare it with one baseline and several unsupervised feature selection methods. The compared algorithms are described as follows.

\begin{enumerate}
\item Using all features (All-Fea): We directly adopt the original features without performing feature selection. This approach is used as a baseline.
\item Max Variance: This is a feature selection method using the classical PCA criteria. Features with maximum variances are chosen for subsequent tasks.
\item Laplacian Score: To best preserve the local manifold structure, feature consistent with Gaussian Laplacian matrix are selected \cite{laplacianscore}. The importance of each feature is determined by its power.
\item SPEC: This is a spectral regression based state-of-the-art feature selection algorithm. Features are selected one by one by leveraging the work of spectral graph theory \cite{SPEC}.
\item MCFS: Features are selected based on spectral analysis and sparse regression problem \cite{MCFS}. Specifically, features are selected such that the multi-cluster structure of the data can be best preserved.
\item UDFS: Features are selected by a joint framework of discriminative analysis and $l_{2,1}$-norm minimization \cite{UDFS}. UDFS selects the most discriminative feature subset from the whole feature set in batch mode.
\end{enumerate}

For each algorithm, all the parameters (if any) are tuned in the range of $\{10^{-6}, 10^{-4}, 10^{-2}, 10^0, 10^2, 10^4, 10^6\}$ and the best results are reported. There are some parameters need to be set in advance. For LS, MCFS and UDFS, we empirically set $k = 5$ for all the datasets to specify the size of neighborhoods. The number of selected features are set as described in Table 1 for all the datasets. For all the compared algorithms, we report the best clustering result with optimal parameters. In the experiments, we utilize K-means algorithm to cluster samples based on the selected features. Note that performance of K-means varies with different initializations. We randomly repeat the clustering 30 times for each setup and report average results with standard deviation.

\subsection{Datasets}

The datasets used in our experiments are described as follows.

\begin{enumerate}
\item Face Image Data: We use three face image datasets for face recognition, namely YaleB \cite{yaleb}, ORL \cite{ORL} and JAFFE \cite{JAFFE}. The YaleB dataset contains 2414 near frontal images from 38 persons under different illuminations. We resize each image to $32 \times 32$. The ORL dataset consists of 40 different subjects with 10 images each. We also resize each image to $32 \times 32$. The Japanese Female Facial Expression (JAFFE) dataset consists of 213 images of different facial expressions from 10 Japanese female models. The images are resized to $26 \times 26$.
\item 3D Motion Data: The HumanEVA dataset is used to evaluate the performance of our algorithm in terms of 3D motion annotation \footnote{http://vision.cs.brown.edu/humaneva/}. This dataset contains five types of motions. Based on the 16 joint coordinates in 3D space, 1590 geometric pose descriptors are extracted using the method proposed in \cite{humaneva} to represent 3D motion data.
\item Object Image Data: We use the Coil20 dataset \cite{COIL20} for object recognition. This dataset includes 1440 grey scale images with 20 different objects. In our experiment, we resize each image to $32 \times 32$.
\item Handwritten Digit Data: We use the USPS dataset to validate the performance on handwritten digit recognition. The dataset consists of 9298 gray-scale handwritten digit images. We resize the images to $16 \times 16$.
\end{enumerate}

\begin{table*}[!ht]
\small
\renewcommand{\arraystretch}{1.3}
\caption{Performance Comparison(ACC $\pm~std$ \%) of All-Fea, MaxVar, LScore, SPEC, MCFS, UDFS and CSPCA. The best results are highlighted in bold. From this table, we can observe that our proposed algorithm has much advantage over other algorithms on all the used datasets.}
\centering
\begin{tabular}{|c||c|c|c|c|c|c|}
\hline
  &  YaleB &  ORL &  JAFFE & HumanEVA &  Coil20 &  USPS \\
\hline \hline
All-Fea & $12.2 \pm 2.8$ & $69.0 \pm 1.7$ & $90.6 \pm 2.8$ & $47.2 \pm 2.0$ & $71.8 \pm 3.7$ & $71.1 \pm 2.7$ \\
\hline
MaxVar & $12.7 \pm 2.6$ & $67.8 \pm 1.9$ & $95.3 \pm 2.1$ & $47.6 \pm 2.7$ & $71.7 \pm 3.2$ & $71.0 \pm 2.8$  \\
\hline
LScore & $12.3 \pm 2.9$ & $70.2 \pm 2.2$ & $94.8 \pm 2.6$ & $47.8 \pm 2.2$ & $71.4 \pm 3.5$ & $73.7 \pm 2.4$  \\
\hline
SPEC & $12.9 \pm 2.5$ & $67.0 \pm 1.8$ & $96.2 \pm 2.7$ & $50.3 \pm 2.5$ & $72.0 \pm 3.8$ & $72.2 \pm 2.5$  \\
\hline
MCFS & $14.5 \pm 2.4$ & $69.4 \pm 1.5$ & $96.4 \pm 1.9$ & $53.6 \pm 2.9$ & $72.4 \pm 3.4$ & $73.1 \pm 2.8$ \\
\hline
UDFS & $15.7 \pm 2.8$ & $69.9 \pm 1.2$ & $96.8 \pm 2.4$ & $56.3 \pm 2.4$ & $73.3 \pm 3.0$ & $73.7 \pm 2.9$ \\
\hline
CSPCA & $\mathbf{19.3 \pm 2.2}$ & $\mathbf{71.3 \pm 1.0}$ & $\mathbf{97.2 \pm 2.2}$ & $\mathbf{56.4 \pm 2.1}$ & $\mathbf{75.2 \pm 2.9}$ & $\mathbf{76.9 \pm 2.1}$  \\
\hline
\end{tabular}
\label{ACC}
\end{table*}

\begin{table*}
\small
\renewcommand{\arraystretch}{1.3}
\caption{Performance Comparison(NMI $\pm~std$ \%) of All-Fea, MaxVar, LScore, SPEC, MCFS, UDFS and CSPCA. The best results are highlighted in bold. From this table, we can observe that our proposed algorithm has much advantage over other algorithms on all the used datasets.}
\centering
\begin{tabular}{|c||c|c|c|c|c|c|}
\hline
  &  YaleB &  ORL &  JAFFE & HumanEVA &  Coil20 &  USPS \\
\hline \hline
All-Fea & $19.3 \pm 3.3$ & $83.7 \pm 2.4$ & $91.7 \pm 3.4$ & $52.2 \pm 4.3$ & $79.1 \pm 5.1$ & $61.7 \pm 3.4$ \\
\hline
MaxVar & $21.3 \pm 2.9$ & $83.0 \pm 2.8$ & $93.0 \pm 3.1$ & $52.7 \pm 3.9$ & $79.5 \pm 4.8$ & $63.7 \pm 3.7$  \\
\hline
LScore & $19.4 \pm 3.5$ & $85.3 \pm 2.6$ & $92.9 \pm 2.9$ & $53.2 \pm 4.1$ & $80.1 \pm 5.2$ & $63.1 \pm 3.2$  \\
\hline
SPEC & $20.3 \pm 3.1$ & $81.6 \pm 2.1$ & $94.7 \pm 3.8$ & $55.4 \pm 3.7$ & $81.2 \pm 5.5$ & $62.3 \pm 3.6$  \\
\hline
MCFS & $22.7 \pm 2.7$ & $83.6 \pm 2.2$ & $95.1 \pm 3.7$ & $57.1 \pm 3.8$ & $82.5 \pm 4.9$ & $63.1 \pm 4.0$ \\
\hline
UDFS & $23.5 \pm 3.2$ & $84.2 \pm 3.1$ & $95.6 \pm 3.3$ & $60.0 \pm 3.4 $ & $83.2 \pm 5.3$ & $63.8 \pm 3.9$  \\
\hline
CSPCA & $\mathbf{29.4 \pm 2.1}$ & $\mathbf{84.7 \pm 2.8}$ & $\mathbf{96.0 \pm 3.8}$ & $\mathbf{62.4 \pm 3.9}$ & $\mathbf{85.1 \pm 5.1}$ & $\mathbf{65.5 \pm 3.1}$  \\
\hline
\end{tabular}
\label{NMI}
\end{table*}

\subsection{Evaluation Metrics}

Following related unsupervised feature selection work \cite{laplacianscore} , we adopt clustering accuracy (ACC) and normalized mutual information (NMI) as our evaluation metrics in our experiments.

Let $q_i$ represent the clustering label result from a clustering algorithm and $p_i$ represent the corresponding ground truth label of arbitrary data point $x_i$. Then $ACC$ is defined as follows:

\begin{equation}
ACC = \frac{\sum_{i=1}^n \delta (p_i, map(q_i))}{ n },
\end{equation}
where $\delta(x, y) = 1$ if $x=y$ and $\delta (x, y) = 0$ otherwise. $map(q_i)$ is the best mapping function that permutes clustering labels to match the ground truth labels using the Kuhn-Munkres algorithm. A larger ACC indicates a better clustering performance.

For any two arbitrary variable $P$ and $Q$, NMI is defined as follows \cite{NMI}:

\begin{equation}
NMI = \frac{I(P, Q)}{\sqrt{H(P)H(Q)}},
\end{equation}
where $I(P, Q)$ computes the mutual information between $P$ and $Q$, and $H(P)$ and $H(Q)$ are the entropies of $P$ and $Q$. Let $t_l$ represent the number of data in the cluster $\mathcal{C}_l(1 \leq l \leq c)$ generated by a clustering algorithm and $\widetilde{t_h}$ represent the number of data points from the $h$-th ground truth class. NMI metric is then computed as follows \cite{NMI}:

\begin{equation}
NMI = \frac{\sum_{l=1}^c \sum_{h=1}^c t_{l,h} log(\frac{n \times t_{l, h}}{•t_l\widetilde{t_h}})}{\sqrt{(\sum_{l=1}^c t_l \log \frac{t_l}{n})(\sum_{h=1}^c \widetilde{t_h} \log \frac{\widetilde{t_h}}{n})}},
\end{equation}
where $t_{l,h}$ is the number of data samples that lies in the intersection between $\mathcal{C}_l$ and $h$th ground truth class. Similarly, a larger NMI indicates a better clustering performance.

\subsection{Experimental Results}

Empirical studies are conducted on six real-world data sets to validate the performance of the proposed algorithm and compare to state-of-the-art algorithms. Table \ref{ACC} and Table \ref{NMI} summarise ACC and NMI comparison results of all the compared algorithms over the used datasets. From the experimental results, we have the following observations.
\begin{enumerate}
\item The feature selection algorithms generally have better performance than the baseline All-Fea, which demonstrates that feature selection is necessary and effective. It can significantly reduce feature number as well as improve the performance.
\item Both SPEC and MCFS utilize a two-step approach (spectral regression) for feature selection. The difference between them is MCFS select features in a batch mode but SPEC conduct this task separately. We can see MCFS gets better results than SPEC because it is a better way to analyze features jointly for feature selection.
\item We can see from the result tables that UDFS gains the second best result, which indicates that it is beneficial to analyze features jointly and simultaneously adopt discriminative information and local structure of data distribution.
\item From the experimental results, we can observe that the proposed CSPCA consistently outperform the other compared algorithms. This phenomenon demonstrate that the proposed algorithm is able to select the most informative features.
\end{enumerate}

\begin{figure*}[!ht]
\centering
\subfigure[]{
\includegraphics[scale=0.18]{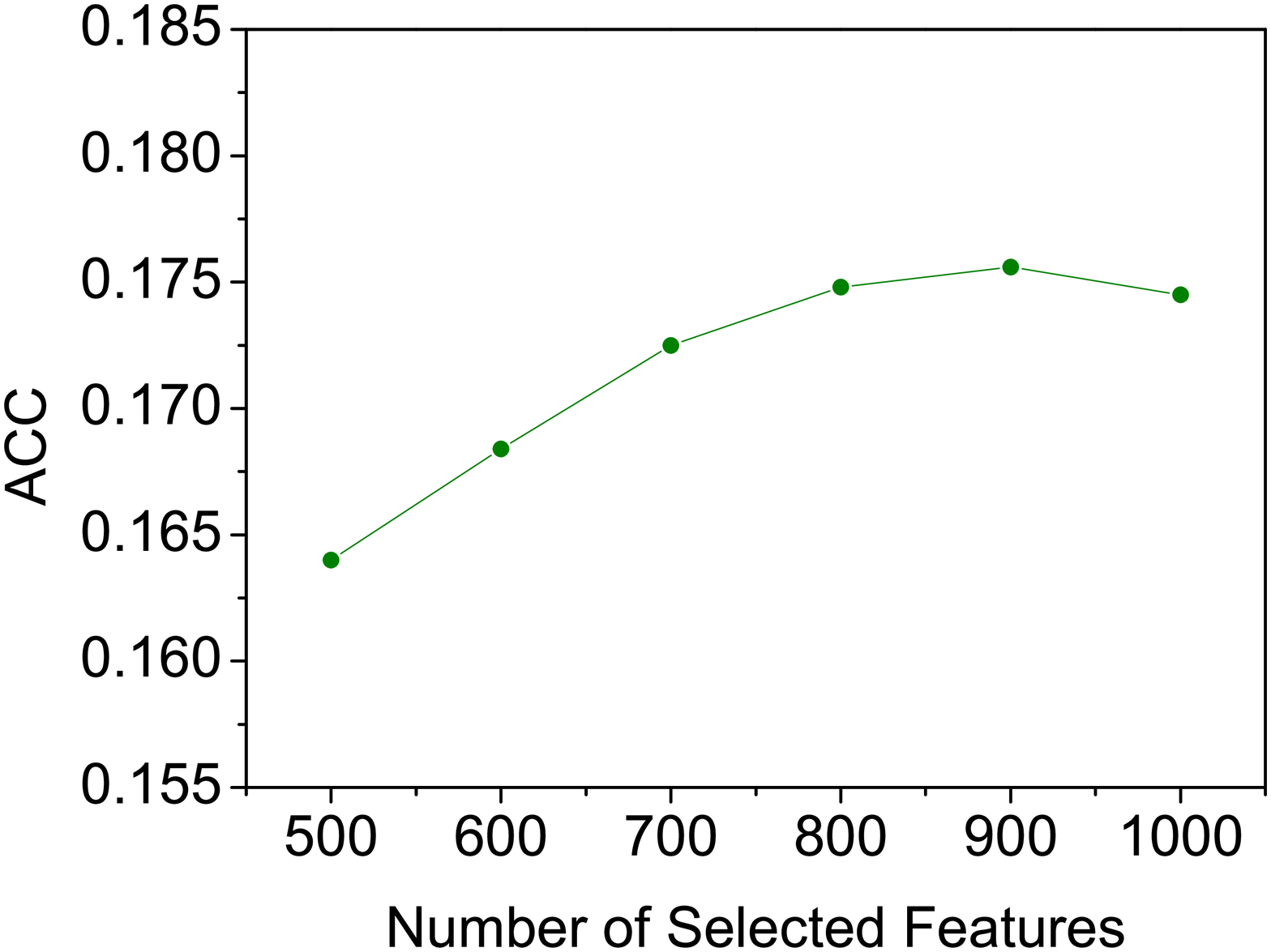}}
\subfigure[]{
\includegraphics[scale=0.18]{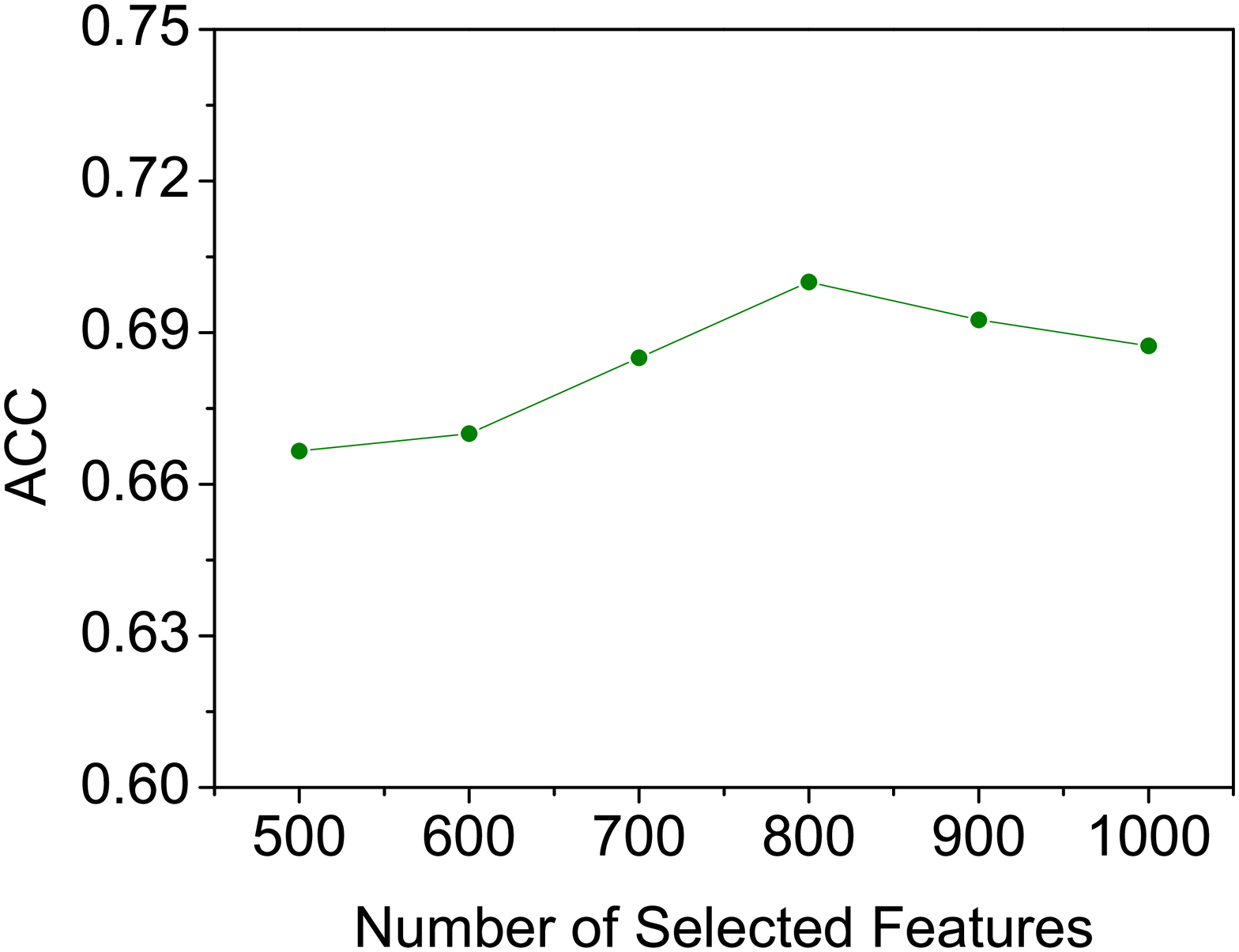}}
\subfigure[]{
\includegraphics[scale=0.18]{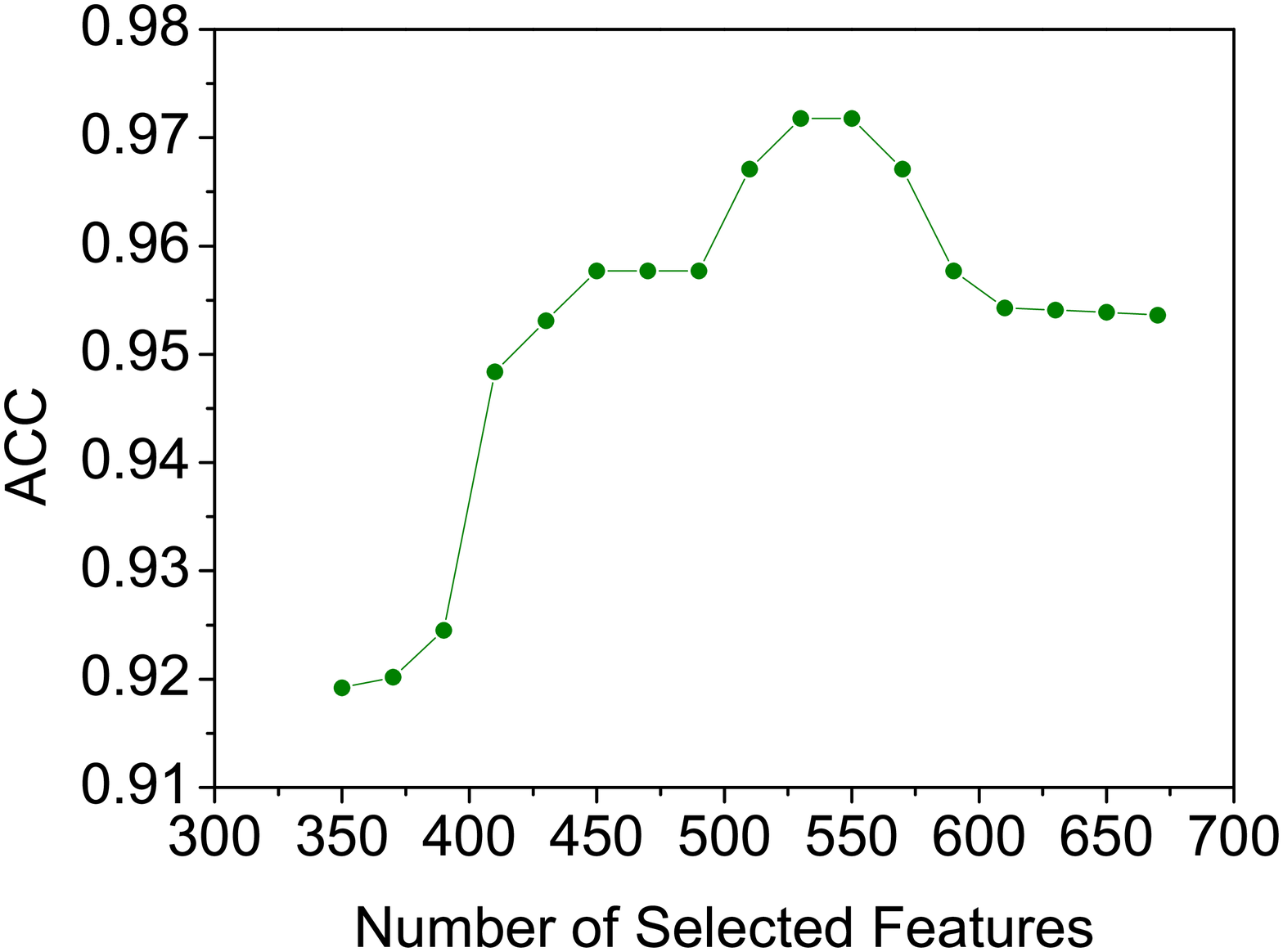}}
\subfigure[]{
\includegraphics[scale=0.18]{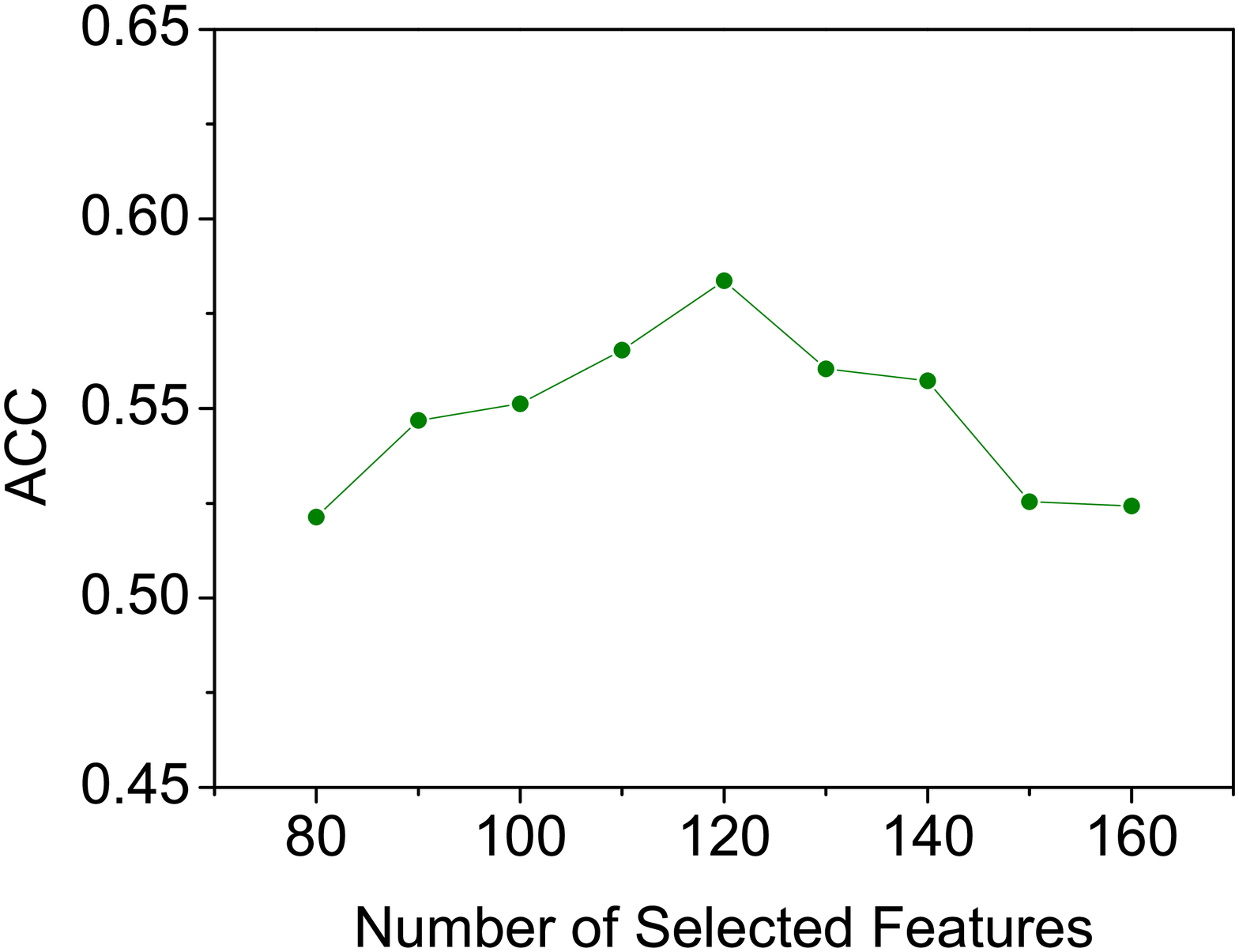}}
\subfigure[]{
\includegraphics[scale=0.18]{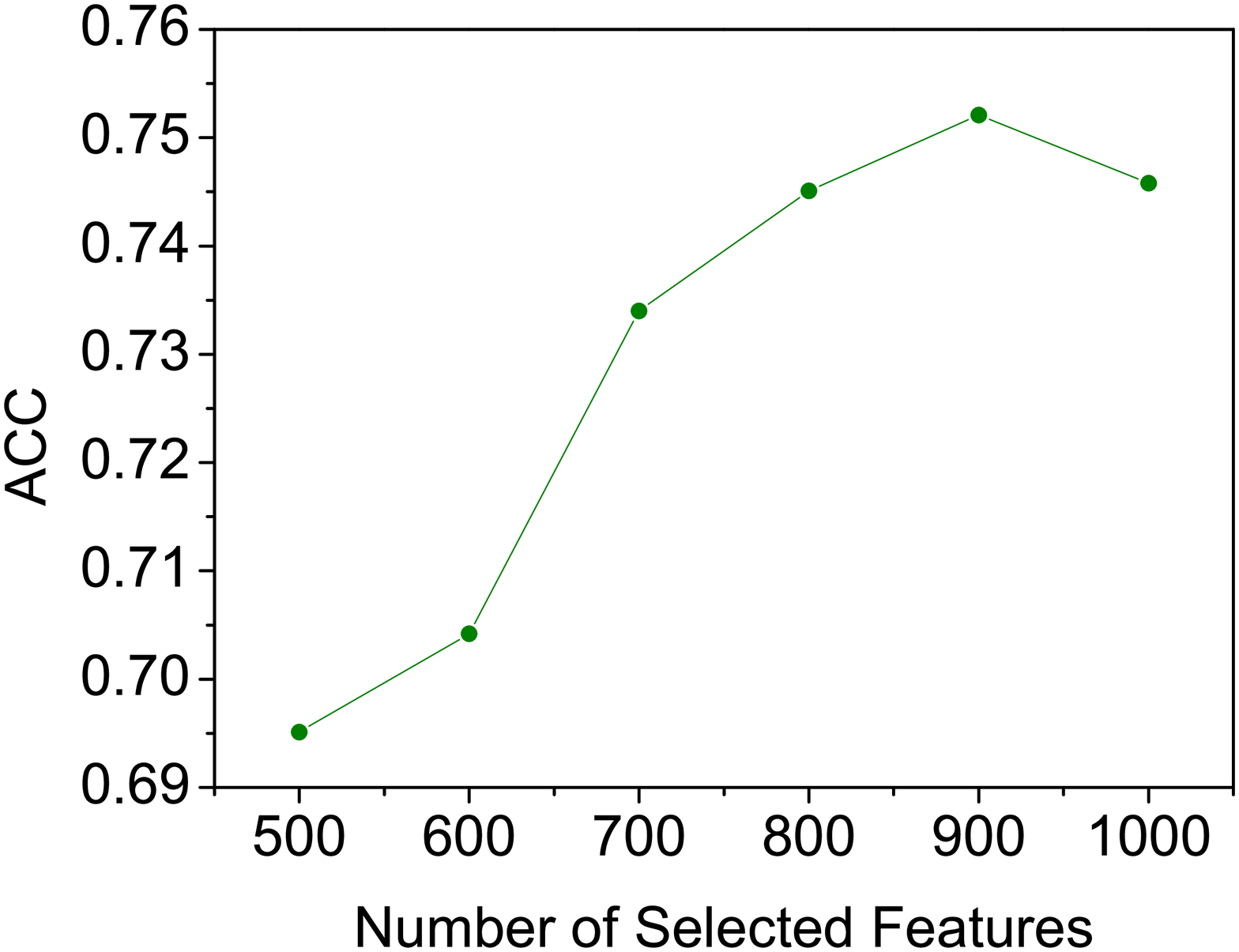}}
\subfigure[]{
\includegraphics[scale=0.18]{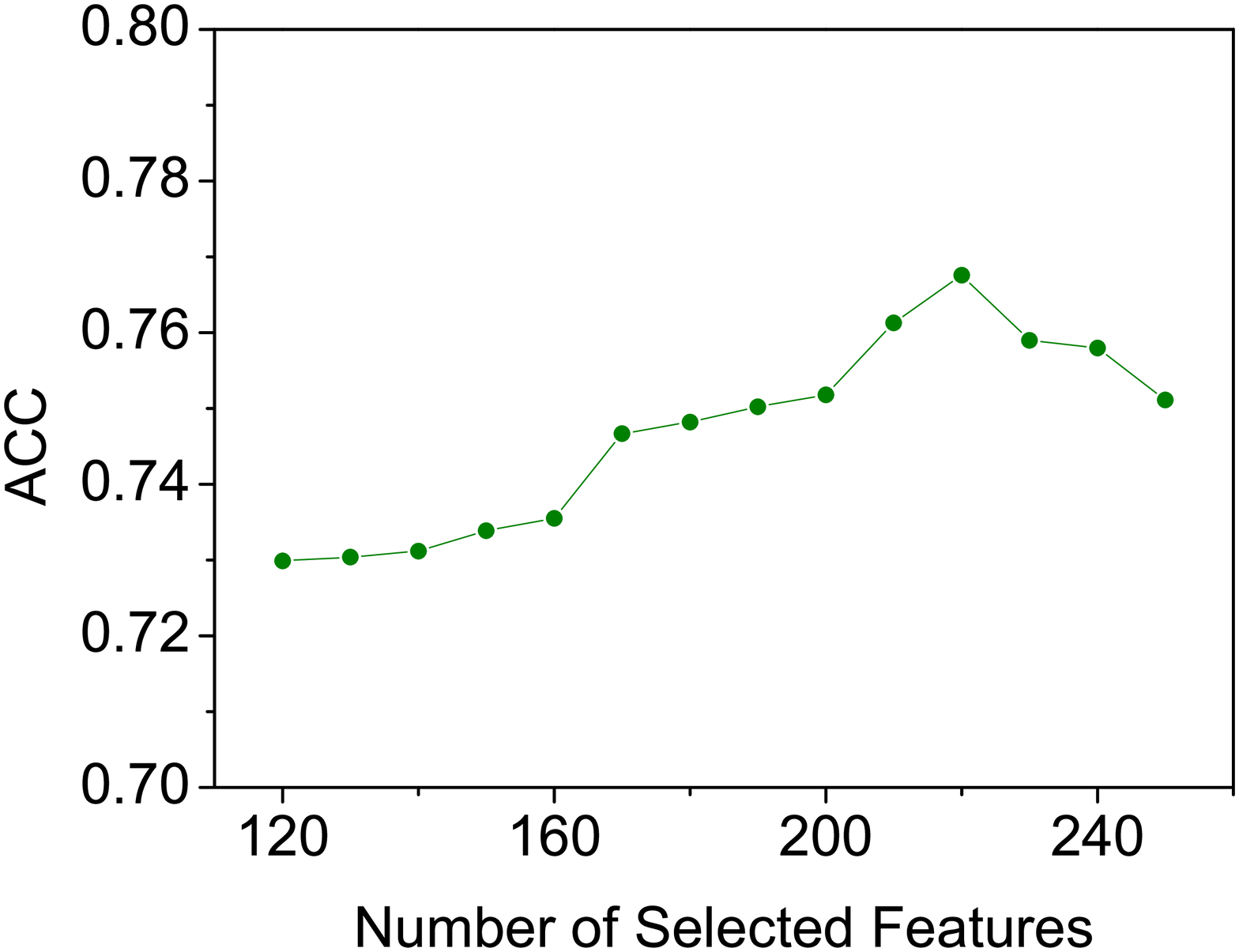}}
\caption{Performance variation w.r.t the number of selected features when we fix $\alpha$ and $\beta$ at 1 using the proposed algorithm. From this figure, we have the following observations: (1) When the number of selected features is too small, the clustering ACC is not competitive with using all features without feature selection. (2) As the number of selected features increases, the clustering ACC rises before its peak in general on all the used datasets. (3) The trend of clustering ACC are varying when different datasets are used. (4) With all the features used, the clustering ACC are generally lower than the peak level on all the datasets. (a) YaleB (b)ORL (c) Jaffe (d) HumanEva (e) Coil20 (f) USPS} 
\label{featureselection}
\end{figure*}

\subsection{Influence of Selected Features}
As the goal of feature selection is to boost accuracy and computation efficiency, experiments are conducted to learn how the number of selected features can affect the clustering performance. From these experiments we can see the general trade-off between performance and computational efficiency over all the used dataset.

Fig. \ref{featureselection} shows the performance variance with right to the number of selected features in terms of clustering ACC. From the results, we have the following observations: 
\begin{enumerate}
\item When the number of selected features is too small, the clustering ACC is not competitive with using all features without feature selection, which is mainly caused by too much information loss. For example, when only 500 features are selected on YaleB, the clustering ACC is relatively low, at only 0.164.
\item As the number of selected features increases, the clustering ACC rises before its peak in general on all the used datasets. How many features are selected to get the peak level is different on different datasets.
\item The trend of clustering ACC are varying when different datasets are used. For example, the clustering ACC keeps stable from using 800 features to using 1000 features for YaleB while drops for the other used datasets. The different variance shown on the six datasets are supposed to be related to the properties of the datasets.
\item After all the features are used (without feature selection), the clustering ACC are generally lower than the peak level on all the datasets. We can safely conclude that as the clustering ACC increases, the proposed algorithm is capable of reducing noise and selecting the most discriminating features.
\end{enumerate} 

\begin{figure*}[!ht]
\centering
\subfigure[]{
\includegraphics[scale=0.13]{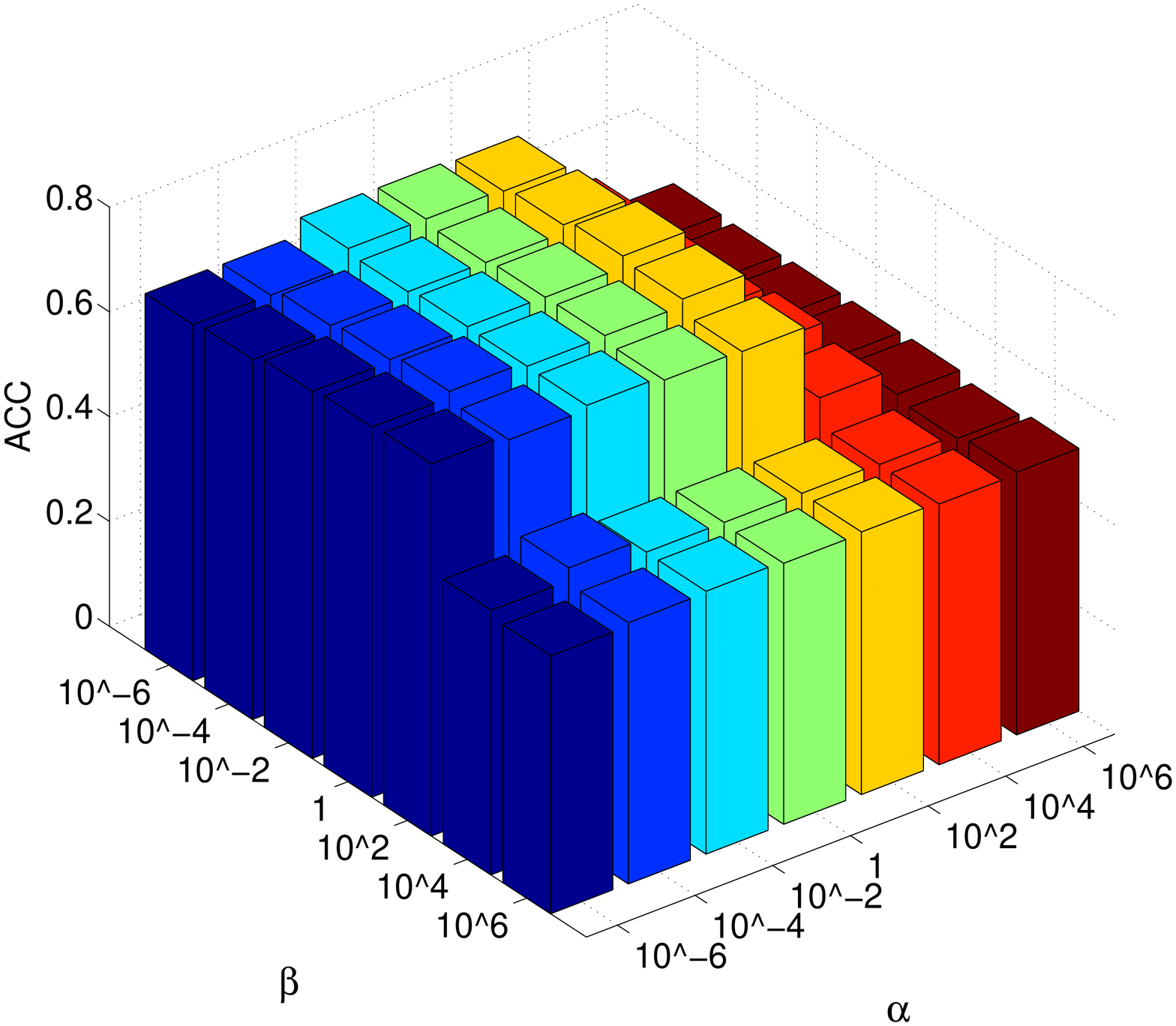}}
\subfigure[]{
\includegraphics[scale=0.13]{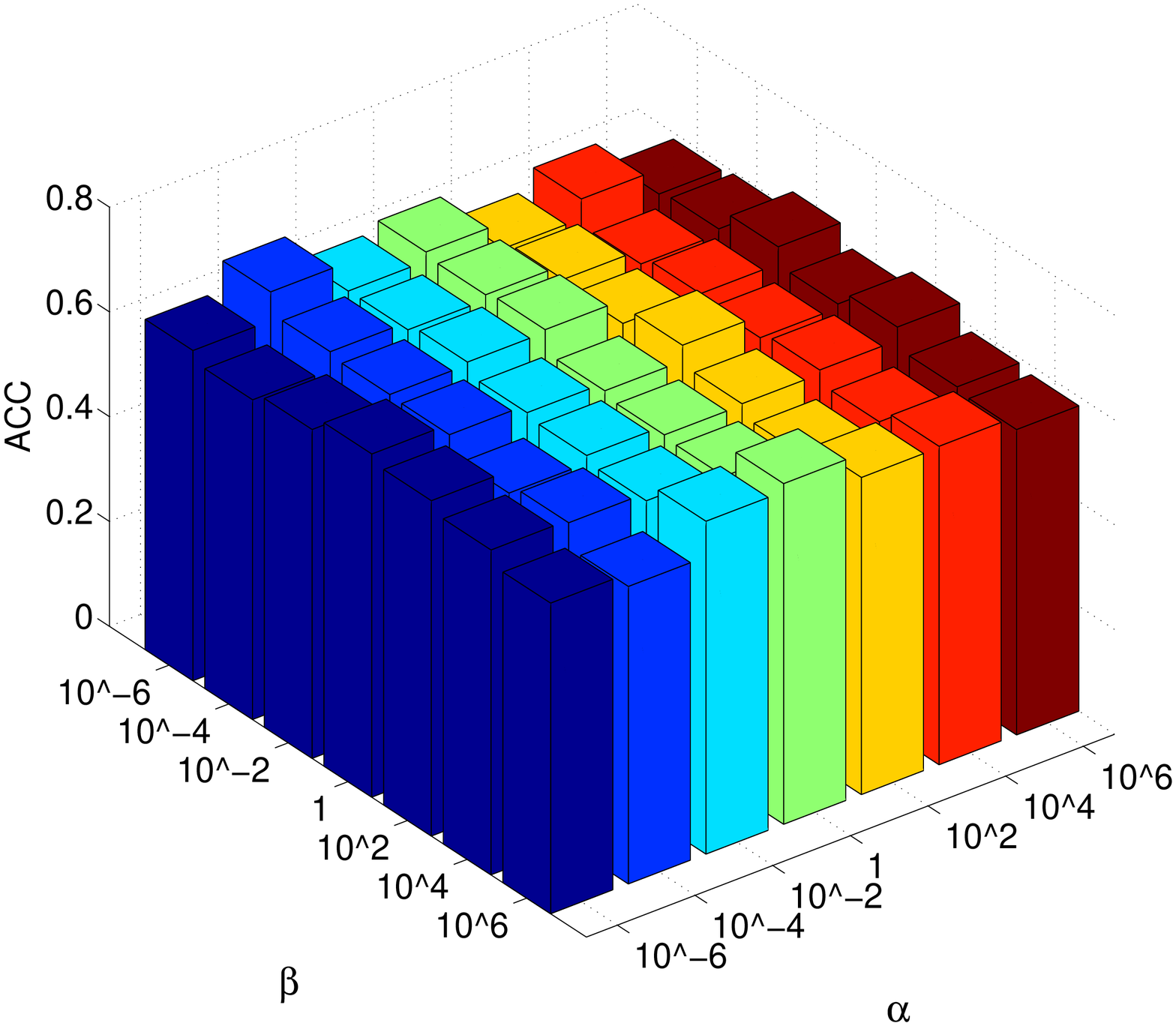}}
\subfigure[]{
\includegraphics[scale=0.13]{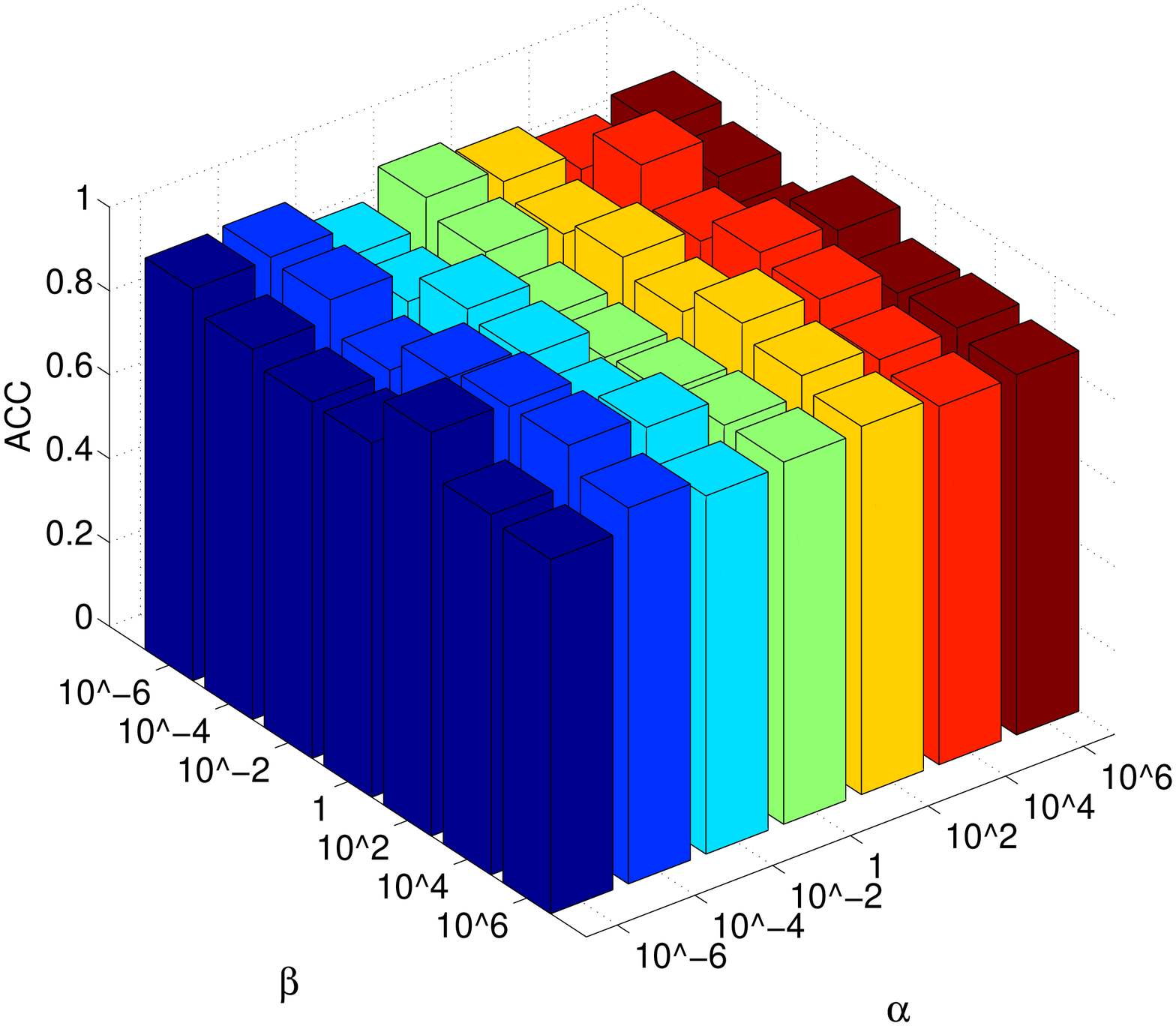}}
\subfigure[]{
\includegraphics[scale=0.13]{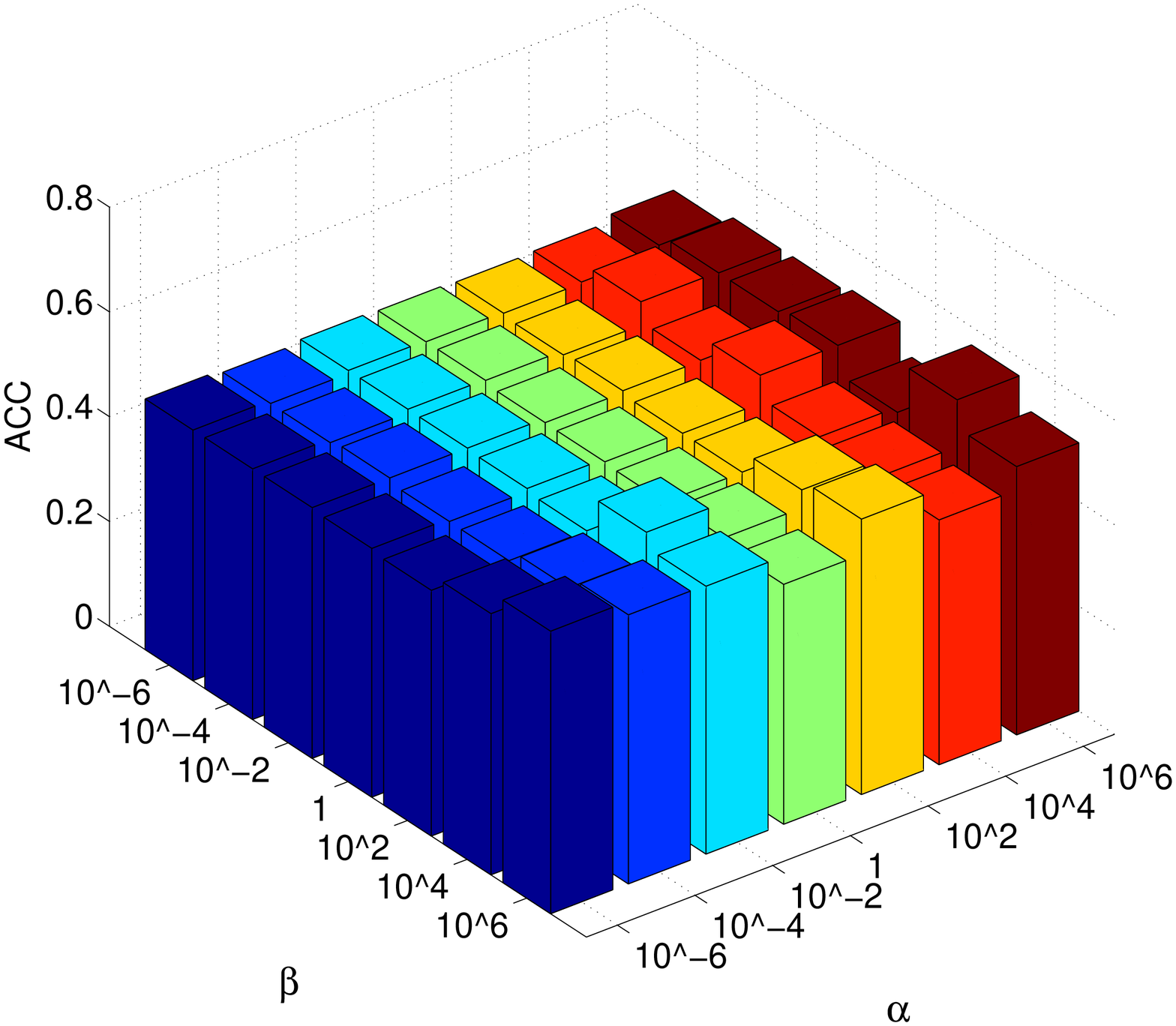}}
\subfigure[]{
\includegraphics[scale=0.13]{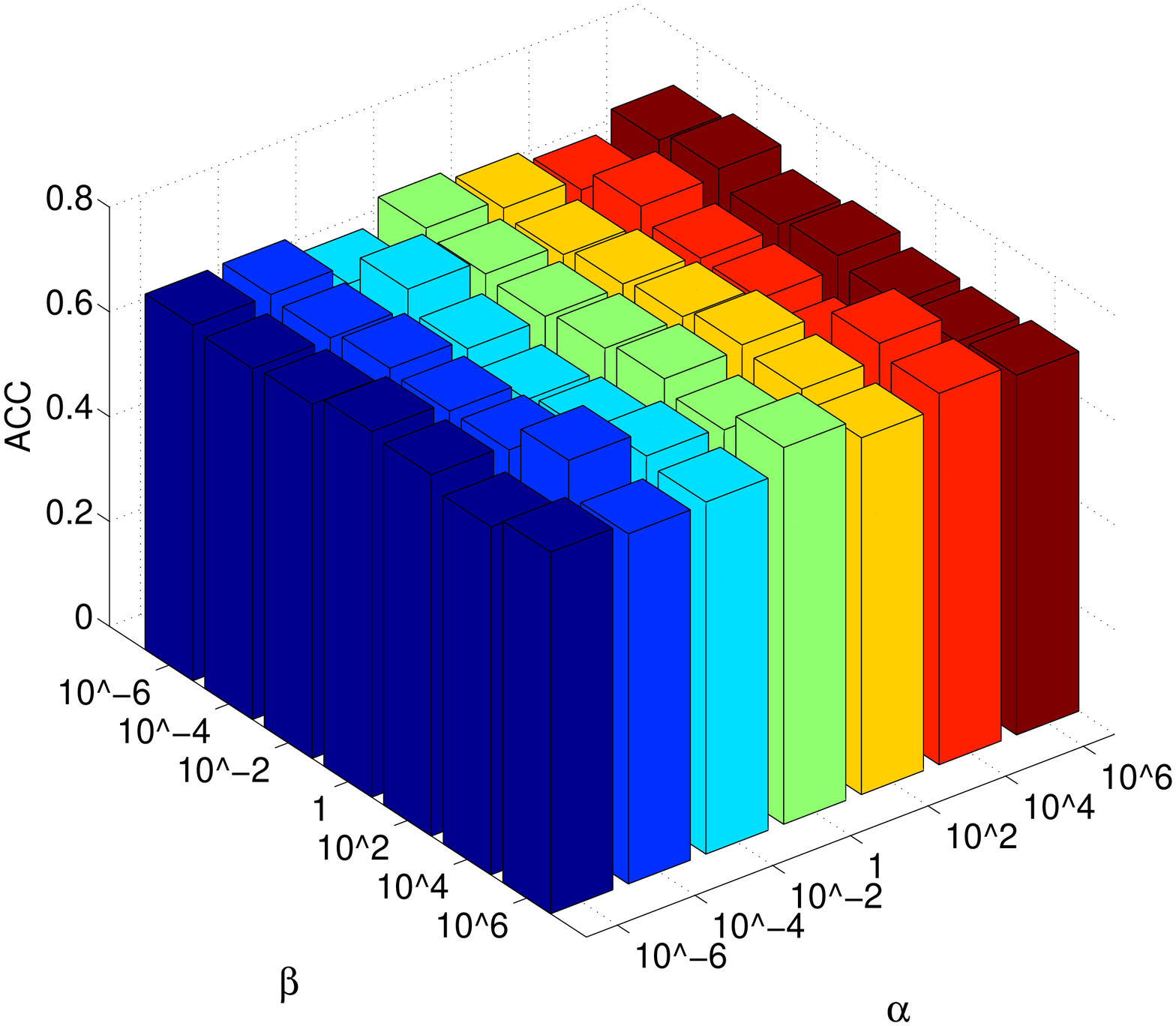}}
\subfigure[]{
\includegraphics[scale=0.13]{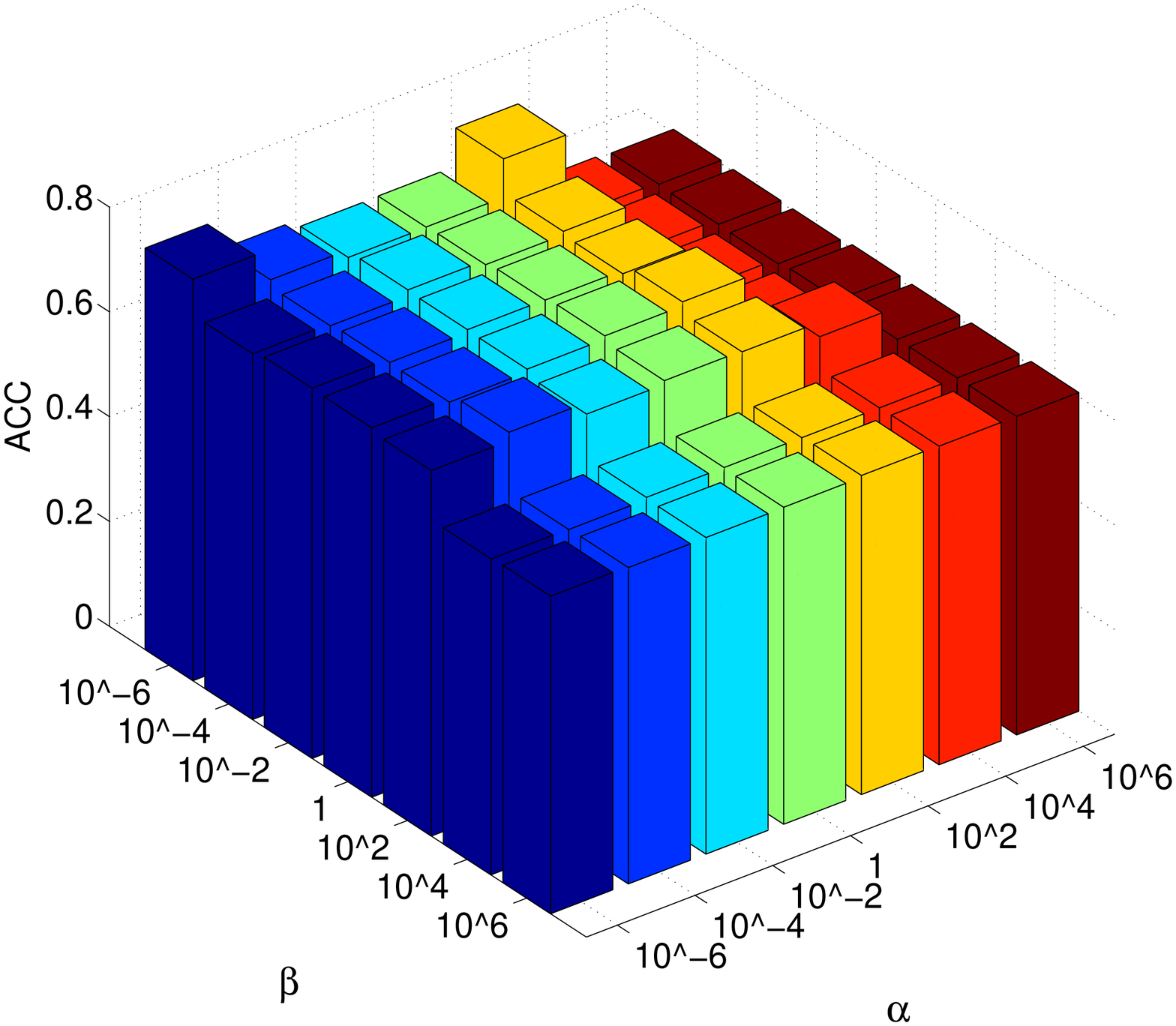}}
\caption{Performance variation w.r.t $\alpha$ and $\beta$ when we fix the number of selected features for clustering. This figure shows different clustering performance when using different values of $\alpha$ and $\beta$. The impact of different combinations of regularization parameters are supposed to be related to the individual properties of the datasets. On the used datasets, we can observe that better experimental results are obtained when the two regularization parameters $\alpha$ and $\beta$ (a) YaleB (b)ORL (c) Jaffe (d) HumanEva (e) Coil20 (f) USPS.} 
\label{ParameterSensitivity}
\end{figure*}
\subsection{Parameter Sensitivity}

Our proposed algorithm involves two regularization parameters, which are denoted as $\alpha$ and $\beta$ in Eq. \eqref{finalobj}. It is beneficial to learn how they influence the feature selection and consequently the performance on clustering. In this section, we conduct several experiments on the parameter sensitivity. We use the clustering ACC to reflect the performance variation.

Fig. \ref{ParameterSensitivity} demonstrates the clustering ACC variation w.r.t $\alpha$ and $\beta$ on the six datasets. From this figure, we learn that the clustering performance changes corresponding to different combinations of $\alpha$ and $\beta$. The impact of different combinations of regularization parameters are supposed to be related to the individual properties of the datasets. On the used datasets, we can observe that better experimental results are obtained when the two regularization parameters $\alpha$ and $\beta$ are comparable. 

\subsection{Performance Variance w.r.t Different Initializations}

In this section, experiments are conducted to evaluate how performance varies when performance variance w.r.t different initializations. Clustering ACC is also used to reflect the performance variation. The Kmeans algorithm has adopted the same initialization. We conduct different initializations, including setting all the diagonal elements of $W$ to 0.5 (1st initialization), 1 (2nd initialization), 2 (3rd initialization), setting all the elements of $W$ to 0.5 (4th initialization), 1 (5th initialization), 2 (6th initialization) and random values (7th initialization). The experimental results are shown in Table \ref{initialization}.

From the experimental results, we can observe that the proposed algorithm always obtains global optima w.r.t different initializations.

\begin{table*}[!ht]
\small
\renewcommand{\arraystretch}{1.3}
\caption{Performance variance w.r.t different initializations, including setting all the diagonal elements of $W$ to 0.5, 1, 2, and setting all the elements of $W$ to 0.5, 1, 2 and random values. In this experiment, the K-means clustering algorithm has adopted the same initialization. It can be seen that our algorithm always converges to the global optima regardless of the different initializations.}
\centering
\begin{tabular}{|c||c|c|c|c|c|c|}
\hline
  &  YaleB &  ORL &  JAFFE & HumanEVA &  Coil20 &  USPS \\
\hline \hline
1st initialization & $19.3$ & $71.3$ & $97.2$ & $56.4$ & $75.2$ & $76.9$ \\
\hline
2nd initialization & $19.3$ & $71.3$ & $97.2$ & $56.4$ & $75.2$ & $76.9$  \\
\hline
3rd initialization & $19.3$ & $71.3$ & $97.2$ & $56.4$ & $75.2$ & $76.9$  \\
\hline
4th initialization & $19.3$ & $71.3$ & $97.2$ & $56.4$ & $75.2$ & $76.9$  \\
\hline
5th initialization & $19.3$ & $71.3$ & $97.2$ & $56.4$ & $75.2$ & $76.9$ \\
\hline
6th initialization & $19.3$ & $71.3$ & $97.2$ & $56.4$ & $75.2$ & $76.9$ \\
\hline
7th initialization & $19.3$ & $71.3$ & $97.2$ & $56.4$ & $75.2$ & $76.9$  \\
\hline
\end{tabular}
\label{initialization}
\end{table*}

\begin{figure*}[!ht]
\centering
\subfigure[]{
\includegraphics[scale=0.18]{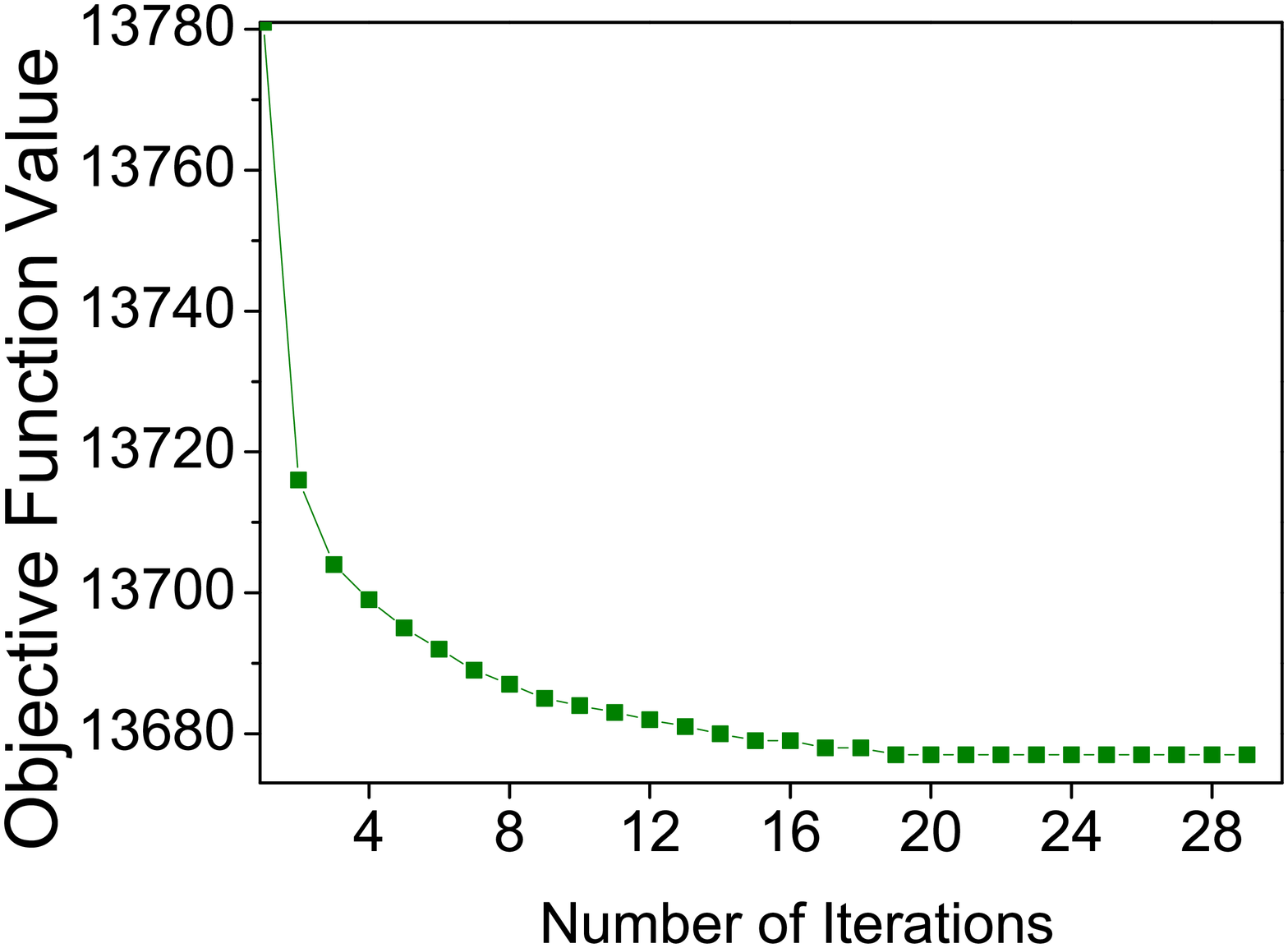}}
\subfigure[]{
\includegraphics[scale=0.18]{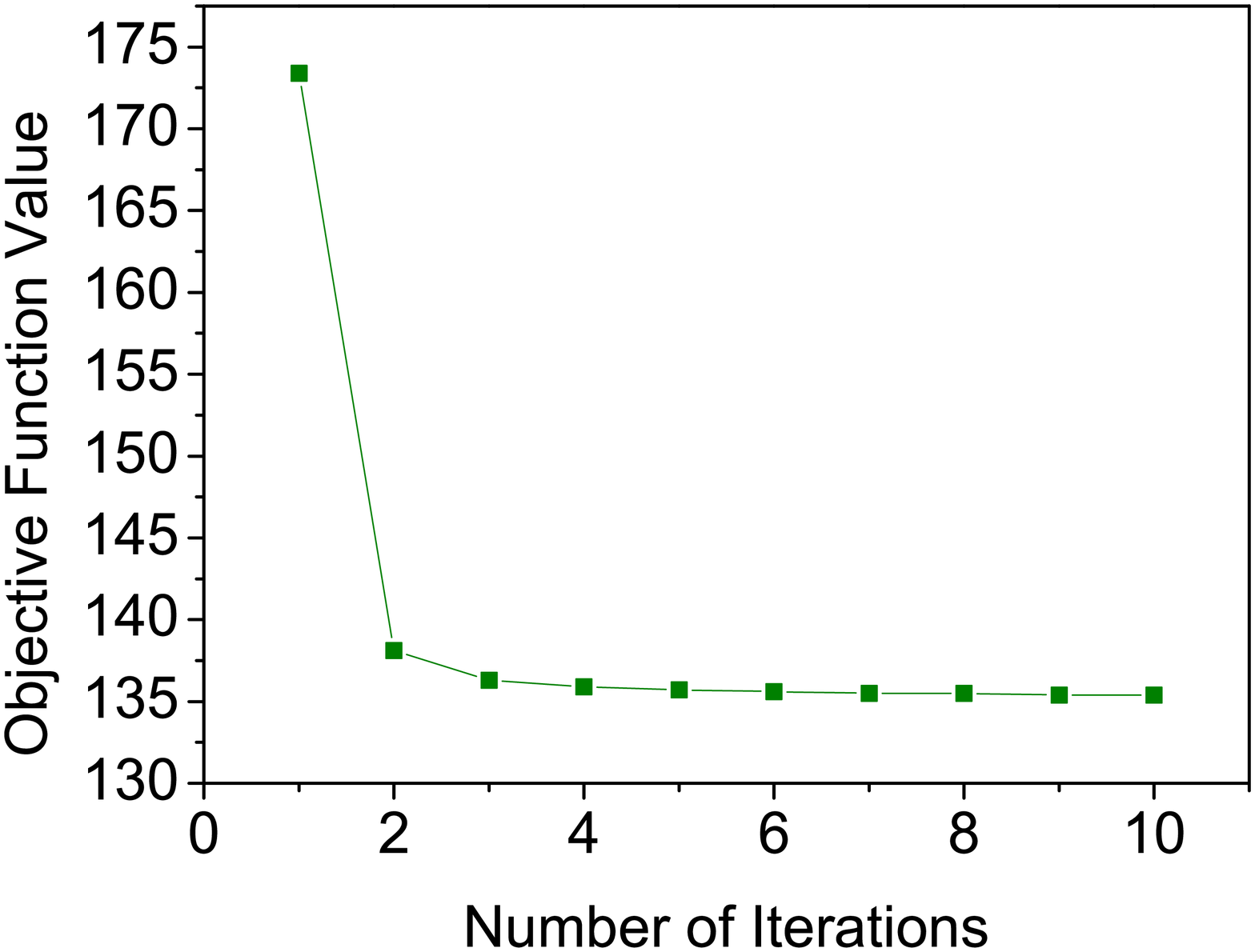}}
\subfigure[]{
\includegraphics[scale=0.18]{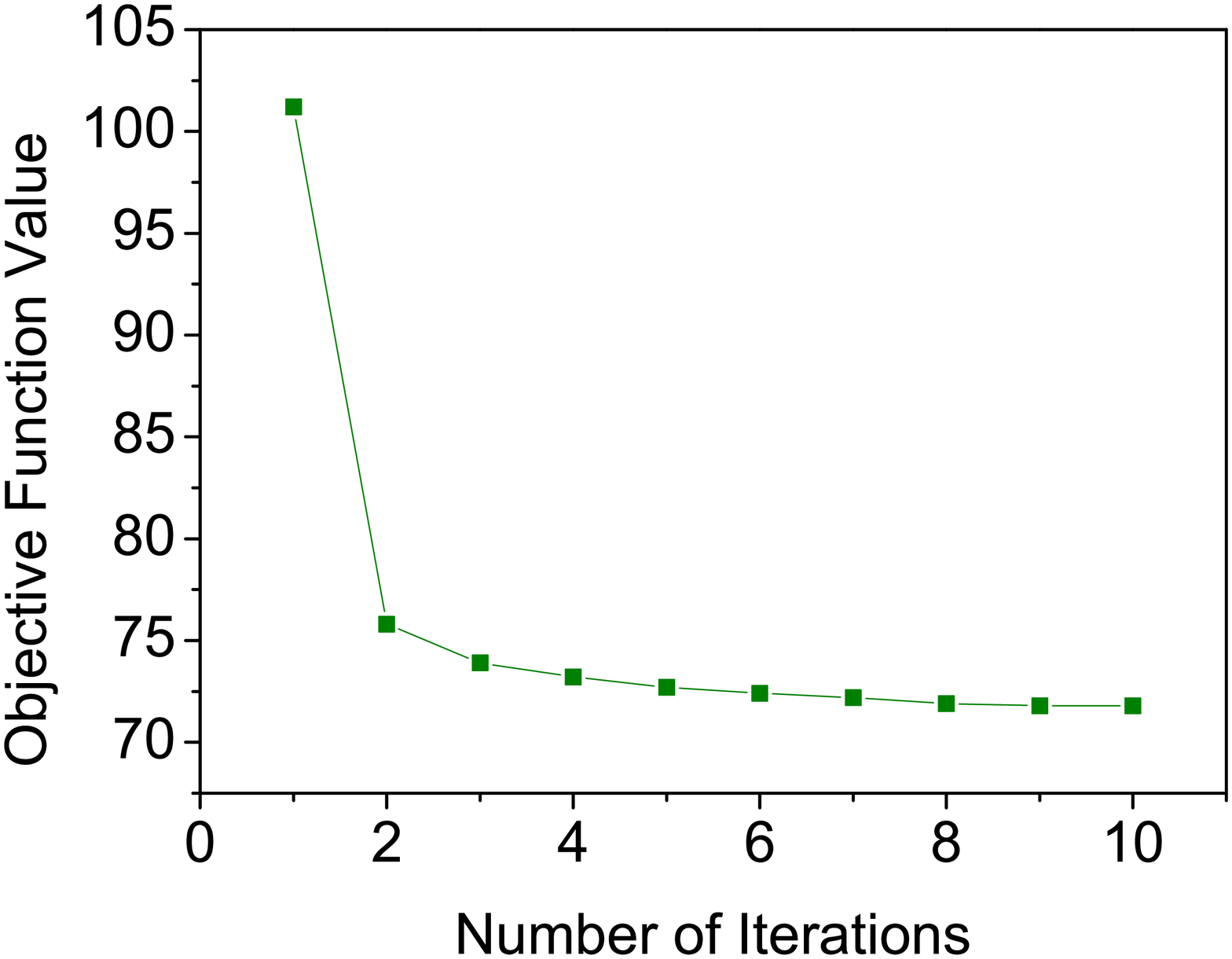}}
\subfigure[]{
\includegraphics[scale=0.18]{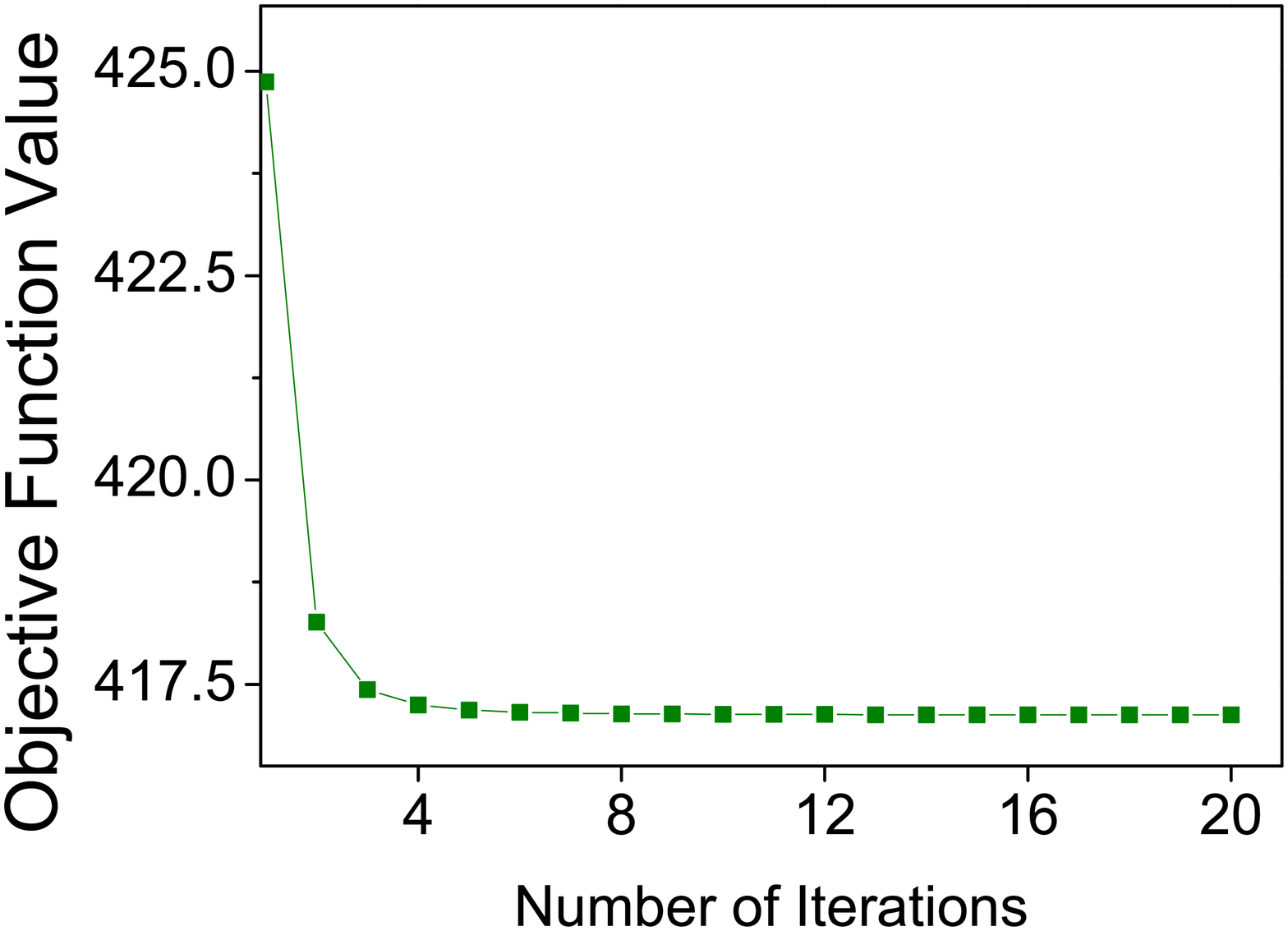}}
\subfigure[]{
\includegraphics[scale=0.18]{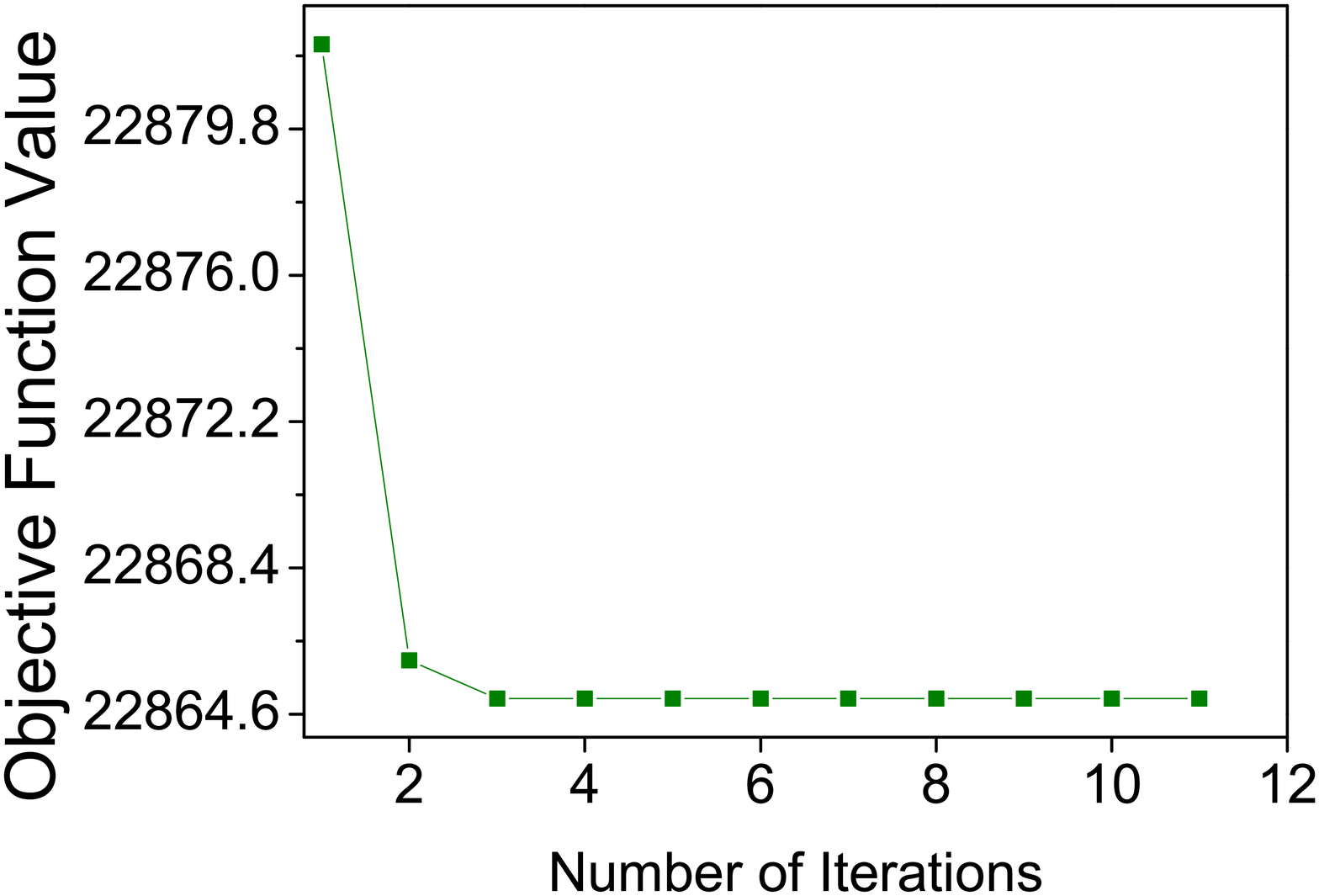}}
\subfigure[]{
\includegraphics[scale=0.18]{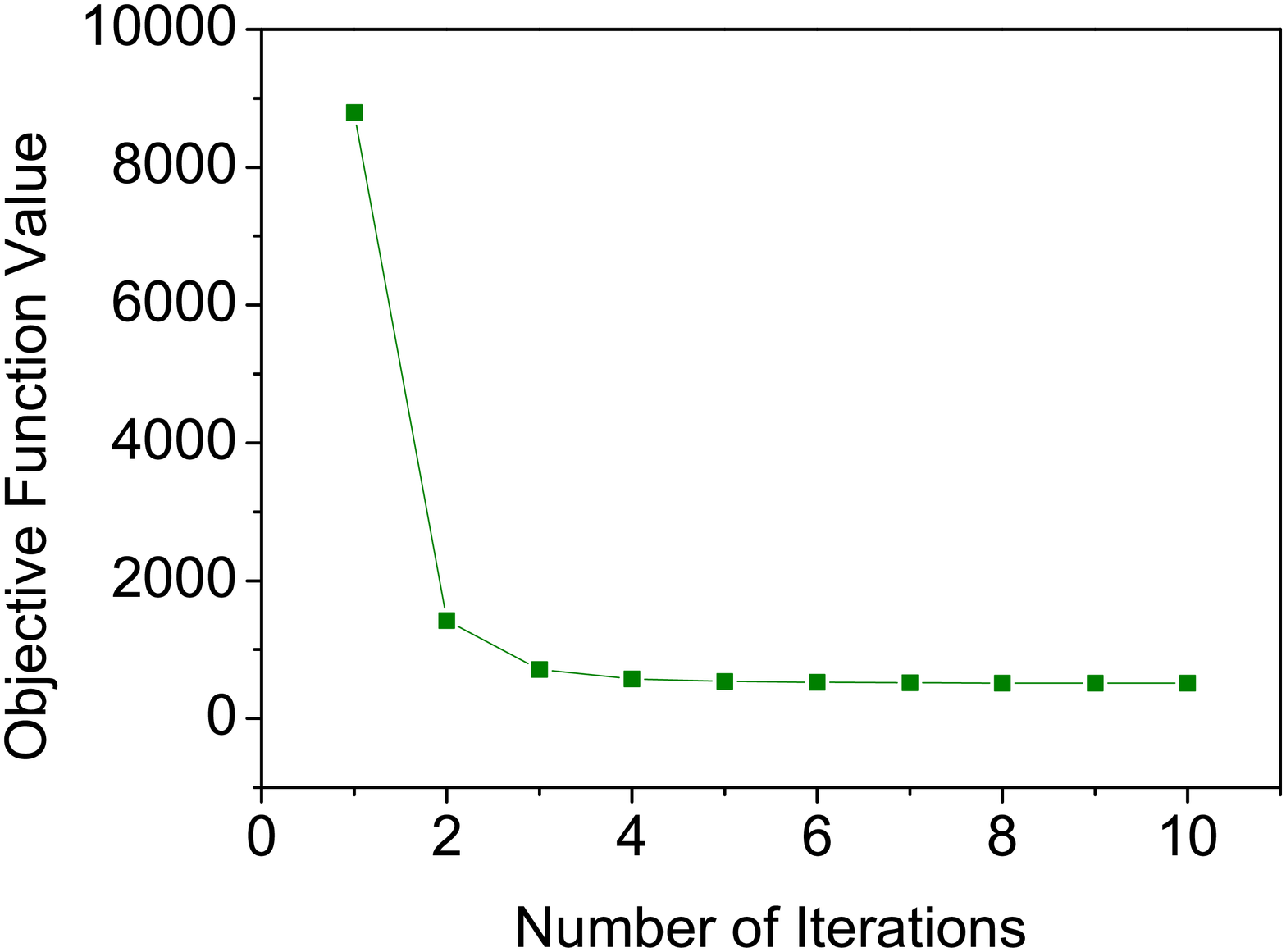}}
\caption{Convergence curves of the objective function value in Eq. \eqref{finalobj} using Algorithm 1. The figure indicates that the objective function value monotonically decreases until convergence by utilizing the proposed algorithm. (a) YaleB (b)ORL (c) Jaffe (d) HumanEva (e) Coil20 (f) USPS.} 
\label{Convergence}
\end{figure*}

\subsection{Convergence Study}
In the previous section, we have proven that the objective function in Eq. \eqref{finalobj} monotonically decreases by using the proposed algorithm. It is interesting to learn how fast our algorithm converges. In this section, we conduct several experiments on validate the convergence of the proposed algorithm. We fix the two regularization parameters $\alpha$ and $\beta$ at 1, which is the median value of the range from which the regularization parameters are tuned.

Fig. \ref{Convergence} shows the convergence curves of the proposed algorithm according to the objective function value in Eq. \eqref{finalobj}. From these figures, we can observe that the objective function value converges quickly. To be more specific, the proposed algorithm can converge within 10 iterations on all the used datasets, which is very efficient.

\section{Conclusion}

In this paper, we have proposed a novel convex sparse PCA and applied it to feature analysis.
We first prove that PCA can be formulated as a low-rank regression optimization problem. We further incorporate the $l_{2,1}$-norm minimization into the proposed algorithm to make the regression coefficients sparse and make the model robust to the outliers. Different from state-of-the-art robust PCA, the proposed algorithm is capable of solving out-of-sample problems. Additionally, we propose an efficient algorithm to optimize the objective function.

To validate the performances of our algorithm for feature analysis, we conduct experiments on six real-world datasets on clustering. It can be seen from the experimental reuslts that the proposed algorithm outperforms the other state-of-the-art unsupervised feature selection as well as the baseline using all features. Therefore, we conclude that the proposed algorithm is a robust sparse feature analysis method, and its benefits make it especially suitable for feature selection.


%

\ifCLASSOPTIONcaptionsoff
  \newpage
\fi

{
\bibliographystyle{IEEEtran}
\bibliography{SPCA}
}

\end{document}